\pdfoutput=1
\documentclass{article}
\usepackage{amsfonts,amssymb,amsthm,amsmath}
\usepackage{mathtools}
\usepackage{iclr2027_conference,times}
\usepackage{subcaption}
\usepackage{xspace}
\usepackage{url}
\usepackage{xcolor}
\usepackage{booktabs}
\usepackage[hidelinks]{hyperref}
\usepackage[capitalise]{cleveref}
\usepackage{optidef}
\usepackage{enumitem}
\newcommand\numberthis{\addtocounter{equation}{1}\tag{\theequation}}
\newtheorem{theorem}{Theorem}
\newtheorem{lemma}{Lemma}
\newtheorem{corollary}{Corollary}
\newtheorem{proposition}{Proposition}
\newtheorem{assumption}{Assumption}

\crefname{assumption}{assumption}{assumptions}
\Crefname{assumption}{Assumption}{Assumptions}

\newcommand{\ito}{It\^{o}\xspace}
\newcommand{\mi}{\mathcal{I}}
\newcommand{\fisher}{\mathcal{J}}
\newcommand{\tr}{\mathrm{Tr}}
\newcommand{\appx}{g^{\circ}}

\newlist{altassump}{enumerate}{1}
\setlist[altassump]{label=(\roman*), ref=\theassumption(\roman*)}
\crefname{altassumpi}{assumption}{assumptions}
\Crefname{altassumpi}{Assumption}{Assumptions}

\iclrfinalcopy
\newcommand{\authorstatements}{}

\begin{document}
\title{
Noise Schedule Design for Diffusion Models: An Optimal Control Perspective
}

\author{  Seo Taek Kong \\
  ECE \& CSL \\
  University of Illinois Urbana-Champaign \\
  \texttt{skong10@illinois.edu} \\
  \And
  Weina Wang \\
  Computer Science Department \\
  Carnegie Mellon University \\
  \texttt{weinaw@cs.cmu.edu} \\
  \And
  R.~Srikant \\
  ECE, CSL \& NCSA \\
  University of Illinois Urbana-Champaign \\
  \texttt{rsrikant@illinois.edu} \\
}

\maketitle
\begin{abstract}
    We develop a principled framework for analyzing and designing noise schedules in diffusion models. We show that one can recast this design problem as an optimal control problem, whose state is the Fisher information of the diffusion process, which evolves according to an ODE, and whose control input is the noise schedule. The objective of the optimal control problem is a functional involving the Fisher information, which is shown to be an upper bound on the Kullback-Leibler sampling error.
    By solving this optimal control problem, we obtain sufficient conditions on noise schedules under which $\tilde{\mathcal{O}}(d/n)$ sampling error is achievable, where $d$ is the data dimension and $n$ is the number of discretization steps. While existing theoretical works also prove that $\tilde{\mathcal{O}}(d/n)$ sampling error bounds are achievable, these results hold for specific noise schedules that differ from those used in practice.
    Under additional assumptions on the data distribution, we derive a family of noise schedules that recovers the functional forms of schedules widely used in practice.
    We show that these schedules achieve the $\tilde{\mathcal{O}}(d/n)$ error bound when their parameters are chosen according to our analysis.
    Our theory offers a principled approach to noise schedule design, yielding explicit functional forms and guidance for selecting their parameters.

\end{abstract}

\section{Introduction}\label{sec:introduction}
Score-based generative models \citep{ho2020denoising,song2021scorebasedgenerativemodelingstochastic,song2020denoising,liu2022flow} have emerged as a leading framework for approximate sampling.
These methods formulate the generative process as the simulation of a time-reversed stochastic differential equation (SDE) driven by a learned score function.
While the training objective is well-understood \citep{song2021maximumlikelihoodtrainingscorebased,chen2022sampling,chen2023scoreapproximationestimationdistribution,zhang2025convergence}, the design of noise schedules governing the forward and reverse processes remains largely heuristic.
Consequently, standard noise schedules \citep{ho2020denoising,nichol2021improveddenoisingdiffusionprobabilistic,jabri2023scalableadaptivecomputationiterative,xu2022geodiff,von-platen-etal-2022-diffusers} are often selected through empirical sweeps rather than theoretical guidance.

Noise schedule design plays a critical role in generating high-quality data using diffusion models \citep{song2021scorebasedgenerativemodelingstochastic,nichol2021improveddenoisingdiffusionprobabilistic,chen2023importancenoiseschedulingdiffusion,reeves2025informationtheoreticproofsdiffusionsampling}.
This importance is partially reflected in existing theory.
When utilizing constant noise schedules and uniform discretization steps, the best-known bounds on the sampling error measured in the Kullback-Leibler (KL) divergence remain sub-optimal in the problem constants \citep{blockGenerativeModelingDenoising2022,chen2022sampling,hchen_improved_score,conforti2024klconvergenceguaranteesscore}.
Specifically, these upper bounds typically scale as $\tilde{\mathcal{O}}(d^2/n)$ or $\tilde{\mathcal{O}}(\max\{\fisher_\star, d\} / n)$, where $\fisher_\star$ is the Fisher information of the data distribution, $d$ is the data dimension, and $n$ is the number of sampling steps.
Because $\fisher_\star$ can be much larger than $d$, these guarantees suffer from severe problem-dependent penalties.

More recent works \citep{benton2024nearly,conforti2024klconvergenceguaranteesscore} obtain improved $\widetilde{\mathcal O}(d/n)$ bounds on the KL divergence by carefully constructing nonuniform discretization grids.
The discretization grids used to establish these guarantees are rarely deployed in practice \citep{von-platen-etal-2022-diffusers}.
Conversely, the noise schedules favored by practitioners lack comparable theoretical guarantees with their conventional parameter choices, leaving a gap between diffusion model theory and practice.
We refer the reader to \Cref{app:related_works} for a more detailed discussion of related work.

In this work, we establish a new theoretical framework for designing noise schedules, which enlarges the class of schedules that are known to achieve $\tilde{\mathcal{O}}(d/n)$ KL error bounds.
Directly minimizing the KL divergence is analytically intractable because the densities are governed by the Fokker-Planck partial differential equation (PDE).
Using the Fisher information dynamics, we derive a variational upper bound on the KL sampling error whose optimization yields noise schedules achieving the $\tilde{\mathcal{O}}(d/n)$ rate.
The bound is established under mild regularity conditions, and we consider additional assumptions on the data distribution to obtain explicit expressions for the resulting schedules.
This family of Affinely Coupled Schedules (ACS) recovers the functional forms of noise schedules widely used in practice and prescribes how their parameters should scale with problem constants and the sampling budget to achieve $\tilde{\mathcal O}(d/n)$ KL error.
Recent work \citep{debortoli2025dimensionfreeerrorestimatediffusion} also considers variational schedule optimization for an objective different from the KL divergence.
Our formulation establishes how noise schedules used in practice can attain the $\widetilde{\mathcal O}(d/n)$ error bound for the more commonly-used KL divergence.

Our contributions are summarized below:
\begin{enumerate}[leftmargin=*]
    \item \textbf{Optimal Control:}
    We formulate noise schedule design as an optimal control problem.
    The cost functional involves a variational bound on the discretization error with a non-trivial Lagrangian that admits path-dependent variations, and it replaces the intractable PDE representation of probability density evolution with a one-dimensional Fisher information ODE.
    In addition to Girsanov's theorem which is now widely used to theoretically study diffusion-based generative AI models \citep{song2021maximumlikelihoodtrainingscorebased,chen2022sampling,liuLetUsBuild2022}, we also use a number of other tools from probability and information theory such as Bochner's formula, concavity of entropy power, and Blachman-Stam inequality to derive the optimal control formulation.

    \item \textbf{Error Bounds for a Class of Noise Schedules:}
    Under mild regularity conditions on the data distribution, we establish that the optimal control construction achieves $\tilde{\mathcal O}(d/n)$ KL sampling error for a class of noise schedules, provided we have knowledge of the Fisher information of the target distribution.

    \item \textbf{Implementable Budget-Dependent Schedule Design:}
    Under additional assumptions on the target distribution, we show that the parameters of the noise schedules can be chosen as a function of $n$ (the computational budget) without knowledge of the Fisher information to achieve $\tilde{\mathcal O}(d/n)$ sampling error.
    
            \end{enumerate}

\paragraph{General notation.}
The notation $\nabla \phi$ is used to denote the gradient when $\phi$ is a scalar-valued function, or the Jacobian when $\phi$ is a vector-valued function.
Throughout, $\lVert \cdot \rVert$ denotes the standard element-wise $\ell_2$ norm: the Euclidean norm for vectors, the Frobenius norm for matrices, and the Hilbert-Schmidt norm for higher-order tensors.
For any probability density $p$ and a vector-valued function $\phi$, let $\lVert \phi \rVert_{L_2 (p)} = \sqrt{\mathbb{E}_{X\sim p} \lVert \phi(X) \rVert^2}$.
We use $\lesssim$ to denote an inequality ignoring absolute (problem-independent) constants in an inequality.
The set of continuous functions defined on the interval $[0, T]$ is denoted by $C([0, T])$.

\section{Preliminaries and Problem Setup}\label{sec:preliminaries}

Diffusion models are based on two processes: the \emph{forward process}, which transforms data into noise, and its time reversal, often referred to as the \emph{reverse process}, which transforms noise back into data.
The \emph{noise schedule} controls the rates at which these transformations occur.

\paragraph{Forward process.}
Let $p_\star$ denote the density function of the target data distribution on $\mathbb{R}^d$.
The forward process is a stochastic differential equation (SDE) of the following form:
\begin{equation}\label{eq:dXt}
    dX_t = - f(t) X_t dt + \sqrt{g(t)} dB_t , \quad 0\le t \le T, \quad X_0 \sim p_\star,
\end{equation}
where $(B_t\colon t\in[0,T])$ is a standard Brownian motion on $\mathbb{R}^d$ independent of $X_0$, and $f(\cdot)$ and $g(\cdot)$ are real-valued functions of time.
Let $p_t$ denote the density function of $X_t$.
We call the function pair $(f,g)$ the \emph{noise schedule}.
Many theoretical studies of diffusion models focus on the simplified setting $f(t)=1$ and $g(t)=2$ for all~$t$, in which case the forward process is the standard Ornstein--Uhlenbeck (OU) process.

To see how the noise schedule controls the distribution of the forward process, note that the marginal distribution of $X_t$ admits the representation
\begin{equation}\label{eq:Xt}
    X_t \overset{d}{=} \alpha_t X_0 + \sigma_t Z ,
    \quad
    Z \sim \mathcal{N}(0, I),
\end{equation}
where
\begin{equation*}
    \alpha_t =  \exp \left(-\int_0^t f(\tilde{t}) d\tilde{t} \right)
    ,
    \quad \sigma_t^2 = \alpha_t^2 \int_0^t \frac{g(\tilde{t})}{\alpha_{\tilde{t}}^2} d\tilde{t}.
\end{equation*}
In the OU setting where $f(t)=1,g(t)=2$, these quantities reduce to $\alpha_t = e^{-t},\sigma_t^2 = 1- e^{-2t}$.
More generally, a larger value of $f(t)$ leads to a faster decay of the contribution of $X_0$ in the distribution of $X_t$, while a larger value of $g(t)$ corresponds to a faster injection of Gaussian noise.
In this paper, we consider the class of noise schedules $\Theta = \mathcal{F} \times \mathcal{G}$ defined as
\begin{equation}\label{eq:schedule_class}
\begin{split}
    \mathcal{F} &= \left\{f \in C([0,T]): f(t)\ge 0 \text{ for all } t \in [0, T] \right\}
    , \\
    \mathcal{G} &= \left\{g \in C([0,T]): g(t)>0 \text{ for all } t \in [0, T] \right\}
    .    
\end{split}
\end{equation}
The noise schedules used in most diffusion model implementations \citep{von-platen-etal-2022-diffusers} fall under the class $\Theta$.

\paragraph{Reverse process.}
The reverse process is the time reversal of the forward process $(X_t\colon t\in[0,T])$ in \eqref{eq:dXt}.
That is, the reverse process $(X_\tau^\leftarrow \colon \tau\in[0,T])$ is defined by $X_\tau^\leftarrow \overset{d}{=} X_{T-\tau},\tau\in[0,T]$.
Then the reverse process admits the following SDE representation \citep{ANDERSON1982313,haussmann1986time}:
\begin{equation}\label{eq:reverse_sde}
\begin{gathered}
    dX_\tau^\leftarrow
    = \left[
        f(T-\tau) X_\tau^\leftarrow + g(T-\tau) s_{T-\tau} (X_\tau^\leftarrow)
    \right] d\tau
    + \sqrt{g(T-\tau)} dB_\tau,\\
    0\le \tau\le T, \quad X_0^\leftarrow \sim p_T,
\end{gathered}
\end{equation}
where $(B_\tau\colon 0\le \tau \le T)$ is a standard Brownian motion on $\mathbb{R}^d$ independent of $X_0^\leftarrow$, and the function $s_t(\cdot)$ defined as $s_t(x) \triangleq \nabla \log p_t(x)$ is called the score function.

\subsection{Noise Schedule Optimization}
To turn the reverse process  \eqref{eq:reverse_sde} into an implementable algorithm for generating data, several approximations need to be made:
(1) The score function $s_t$ is unknown since it depends on the unknown data distribution, and is therefore replaced by an approximation $\hat{s}_t$, typically learned through a training process called score matching;
(2) The reverse SDE cannot be solved exactly, but rather must be solved numerically with time discretization;
(3) The reverse SDE cannot be initialized exactly from $p_T$, and is instead initialized from a Gaussian distribution so the initial state can be easily sampled.

A central question in the analysis of diffusion models is how these approximations affect the error between the distribution of the generated samples and the target distribution \citep{blockGenerativeModelingDenoising2022,chen2022sampling,benton2024nearly,conforti2024klconvergenceguaranteesscore}.
Existing analysis shows that the overall error can be separated into three additive terms: (1) score-matching error, (2) discretization error, and (3) initialization error, corresponding to the three approximations above.
The score-matching error arises from the training stage, while the discretization error and the initialization error arise from the sampling stage when one solves the reverse SDE to generate data samples.

In this work, we establish a noise schedule design framework for sampling. 
When using diffusion models for generative tasks, a score function is first learned during a training phase.
The trained network can then be reused with different noise schedules without changing its parameters, provided that the training phase covers the required SNR range as in \citep{karrasElucidatingDesignSpace2022}.
This allows us to study noise schedule design separately from the training procedure. 
In the main body of the paper, we focus on how noise schedules affect initialization and discretization errors assuming that the score function is known perfectly. 
Since the score matching error enters the sampling error bound as an additive term, it is easy to extend the results to the case where there is an error in the score function estimation; see \Cref{app:score_matching}. 

Specifically, we consider the following method for solving the reverse SDE.
We select discretization time points $0=\tau_0<\tau_1<\tau_2<\dots<\tau_n=T$, where $\tau_k = kh,k=0,1,2,\dots,n$ with a step size~$h>0$, and solve the following SDE:
\begin{equation}\label{eq:exponential_integrator}
\begin{gathered}
    d \hat{X}_\tau^\leftarrow = \left[f(T-\tau) \hat{X}_{\tau}^\leftarrow + g(T-\tau) s_{T-\tau_k} (\hat{X}_{\tau_k}^\leftarrow) \right] d\tau + \sqrt{g(T-\tau)} dB_\tau,\\
    \tau \in [\tau_k, \tau_{k+1}),\quad k=0,1,2,\dots,n-1,\quad \hat{X}_0^\leftarrow \sim \mathcal{N}(0, \sigma_T^2 I).
\end{gathered}
\end{equation}
Note that here we take the approach where we use a constant step $h$ and let the noise schedule $(f,g)$ vary over time.
This is without loss of generality, as discussed in \Cref{app:constant_step_size}.
We allow the schedule $(f, g) \in \Theta$ to depend on the number of discretization steps $n$.

Let $\hat{p}_\star$ denote the density of $\hat{X}_{T}^\leftarrow$, i.e., the output of the approximated reverse SDE in \eqref{eq:exponential_integrator}.
Recall that $p_\star$ is the density of the data distribution.
Then our goal is to minimize the Kullback-Leibler (KL) divergence between $p_\star$ and $\hat{p}_\star$, i.e., to solve the following problem
\begin{mini}
    {(f, g) \in \Theta}
    {\mathrm{KL}(p_\star \| \hat{p}_\star)}     {\label{eq:objective}}
    {}
    \addConstraint{p_\star}{\text{ is the density of } X_T^\leftarrow \text{ in \eqref{eq:reverse_sde}}}
    \addConstraint{\hat{p}_\star}{\text{ is the density of } \hat{X}_T^\leftarrow \text{ in \eqref{eq:exponential_integrator},}}
\end{mini}
where the constraints are specified by partial differential equations (PDEs) for the density functions of $X_t^\leftarrow$ and $\hat{X}_t^\leftarrow$ given by the Fokker--Planck equations.
Solving optimization problems subject to PDE constraints is notoriously difficult \citep{hinze2008optimization,risken1989fokker}, especially in high dimensions.
In this work, we present a principled framework to find the optimal noise schedule while circumventing the difficulty involved with solving this PDE-constrained optimization problem.

\paragraph{Regularity assumption on data distribution.}
For convenience, we use $X_\star \sim p_\star$ to denote a sample from the data distribution with density $p_\star$.
We also generally use the subscript $\star$ for quantities associated with the data distribution.
Let $s_\star\triangleq \nabla \log p_\star$ be the score function of $p_\star$.
In our analysis, we frequently use the \emph{Fisher information} of a distribution.
For a density $p$ with score function $s$, its Fisher information is defined as
\begin{equation}
    \fisher(p) \triangleq \lVert \nabla \log p \rVert_{L_2 (p)}^2 = \lVert s \rVert_{L_2 (p)}^2
    .
\end{equation}
The Fisher information of the data distribution is denoted by $\fisher_\star \coloneqq \lVert s_\star \rVert_{L_2}^2$.
\begin{assumption}\label{ass:fisher_information}\label{ass:higher_order}\label{ass:first}
We assume that $\lVert X_\star \rVert_{L_2}^2$, $\lVert s_\star \rVert_{L_2}^2$, $\lVert \nabla s_\star \rVert_{L_2}$, and $\lVert \nabla^2 s_\star \rVert_{L_2}$ are all finite.
\end{assumption}
The $L_2$ norms in \Cref{ass:fisher_information} are all with respect to $p_\star$.
When \Cref{ass:fisher_information} is not satisfied, early stopping can be applied where the target distribution $p_\star$ is replaced by $p_\delta$ for some $\delta > 0$ as in \citep{chen2022sampling,benton2024nearly,bortoliConvergenceDenoisingDiffusion2023,lee2022convergencescorebasedgenerativemodeling}.

\section{Main Results}\label{sec:main_results}

Our main results consist of two parts.
In Section~\ref{sec:opt-control}, we present an optimal control formulation for the noise schedule design problem, which leads to a class of schedules that achieve $\tilde{\mathcal{O}}(d/n)$ sampling error.
In Section~\ref{sec:schedules}, we derive noise schedules with closed-form expressions under additional assumptions on the target data distribution.
We will provide proof sketches of these results in \Cref{sec:analysis}, and defer the full proofs to \Cref{app:main_results,app:schedules}.

\subsection{An Optimal Control Formulation for Noise Schedule Design}
\label{sec:opt-control}

Our objective is to solve the optimization problem \eqref{eq:objective}, which minimizes $\mathrm{KL}(p_\star \| \hat{p}_\star)$, the KL-divergence between the data distribution $p_\star$ and the output distribution $\hat{p}_\star$.
Here, the densities $p_\star$ and $\hat{p}_\star$ are functions on the high-dimensional space $\mathbb{R}^d$, and they depend on the noise schedule $(f,g)$ through PDEs given by the Fokker--Planck equations.
This PDE dependence makes directly solving the KL-minimization problem analytically intractable.

To circumvent this difficulty, a key idea in our approach is to work with the Fisher information $\fisher_t = \fisher(p_t)$.
Unlike the density $p_t$, the Fisher information $\fisher_t$ is a scalar quantity.
Thus, rather than having its dynamics described by a PDE as $p_t$ does, $\fisher_t$ evolves according to an ODE, derived using the heat flow and the Bochner formula (\Cref{lem:fisher_ode}):
\begin{equation}\label{eq:fisher_ode}
    \dot{\fisher}_t = 2 f(t) \fisher_t - g(t) \lVert \nabla s_t \rVert^2_{L_2}.
\end{equation}

\Cref{thm:ub} below shows that the KL objective can be upper bounded in terms of $\fisher_t$.
\begin{theorem}\label{thm:ub}
    Suppose \Cref{ass:fisher_information} holds.
    For any noise schedule $(f, g) \in \Theta$ and any sufficiently small step size $h$, it holds that   \begin{equation}\label{eq:sampling_error}
        \mathrm{KL}(p_\star \| \hat{p}_\star)
        \lesssim
        \frac{\alpha_T^2}{ \sigma_T^2} \lVert X_\star \rVert^2_{L_2 (p_\star)}
        +
        h d \int_0^T \left( 2 f(t) - \frac{\dot{\fisher}_t}{\fisher_t} \right)^2 dt
        .
    \end{equation}
\end{theorem}

\Cref{thm:ub} allows us to formulate the noise schedule design problem as the following optimal control problem:
\begin{mini}
    {(f, g) \in \Theta}
    {\frac{\alpha_T^2}{\sigma_T^2} \lVert X_\star \rVert^2_{L_2(p_\star)} + h d \int_0^T \left(2 f(t) - \frac{\dot{\fisher}_t}{\fisher_t} \right)^2 dt}     {\label{eq:optimal_control}}
    {}
    \addConstraint{\dot{\fisher}_t}{= 2 f(t) \fisher_t - g(t) \lVert \nabla s_t \rVert^2_{L_2}}
    \addConstraint{\fisher_0}{= \fisher_{\star} , \quad \frac{\alpha_T^2}{\sigma_T^2} = \frac{d}{\| X_\star \|_{L_2}^2 n} .}
\end{mini}

This optimal control problem \eqref{eq:optimal_control} is significantly more tractable thanks to the ODE constraint.
Moreover, we show in \Cref{thm:sampling_error} that its solution, although obtained by minimizing an upper bound on the KL objective, indeed achieves $\tilde{\mathcal{O}}(d/n)$ sampling error in the KL divergence.
\begin{theorem}\label{thm:sampling_error}
    Suppose \Cref{ass:fisher_information} holds.
    For any sufficiently large $n$, let $f \in \mathcal F$ satisfy $f_{\max} \leq (2T)^{-1}\log(n)$.
    Then, there exists $g^* \in \mathcal{G}$ such that
    \begin{equation}\label{eq:sampling_error_2}
        \mathrm{KL}(p_\star \| \hat{p}_\star)
        \lesssim
        \frac{d}{n} \log^2 \left(\frac{\lVert X_\star \rVert_{L_2}^2}{d} \cdot \frac{\fisher_\star}{d} \cdot n \right).
    \end{equation}
\end{theorem}
Compared with the constant-schedule bounds discussed in \Cref{sec:introduction}, this bound improves the dependence on $d$ and $\fisher_{\star}$.
Consistent with the best-known bounds for the constant schedule, we show in \Cref{app:constant_error_bound} that under our framework, an optimized bound for VP-constant schedules retains a linear dependence on $\fisher_\star$, even when constants depend on problem constants or $n$.

The proof of \Cref{thm:sampling_error} reveals that $g^*$ takes the form
\begin{equation}\label{eq:gstar_form}
    g^*(t) = \lambda \frac{\fisher_t^*}{\lVert \nabla s_t \rVert^2_{L_2}} ,
\end{equation}
where $\lambda$ is a constant independent of $t$ and $(\fisher_t^*: t \in [0, T])$ is the optimal Fisher information trajectory.
The objective in \eqref{eq:optimal_control} is invariant to $f$ when $g$ is chosen as $g = g^*$.
Therefore, one can choose $f$ freely (e.g., as a function of $n$), subject to the constraints of \Cref{thm:sampling_error}, and then choose the corresponding $g^*$ that solves \eqref{eq:optimal_control}.
The choice of $f$ is important in how it affects the resulting form of $g^*$.
Some choices of $f$ yield schedules that are easier to implement, as we show in \Cref{cor:optimal_schedules_practice}.

\subsection{Noise Schedule Design}\label{sec:schedules}
In this section, we consider target data distributions that satisfy the following additional assumption.
\begin{assumption} \label{ass:surrogate}
    The target data distribution is either strongly log-concave (SLC), or belongs to a class of isotropic Gaussian mixture models (GMMs).
    Specifically, the density $p_\star$ of the data distribution satisfies one of the following conditions:
    \begin{altassump}
        \item (SLC)
        \label{ass:logconcave}
        For some $m_\star, M_\star > 0$,
        \begin{equation}\label{eq:slc}
            m_\star I \preceq - \nabla^2 \log p_\star \preceq M_\star I .
        \end{equation}
        \item (GMM) \label{ass:gmm}
        For some mixture weights $\{\pi_i\}_{i=1}^L$, means $\{\mu_i\}_{i=1}^L$ with $\mu_i\in \mathbb{R}^d$, and $\nu > 0$, the density         $p_\star(x)=\sum_{i=1}^L \pi_i \varphi(x;\mu_i, \nu^2 I)$,
        where $\varphi(x;\mu_i, \nu^2 I)$ is the density of $\mathcal{N}(\mu_i, \nu^2 I)$.
            \end{altassump}
\end{assumption}
Although real-world data distributions are highly complex and may not satisfy \Cref{ass:surrogate}, we use this assumption to derive an implementable class of noise schedules.
Under \Cref{ass:surrogate}, we show in \Cref{lem:score_fisher_equivalence} that the ratio $\fisher_t^2 / \lVert\nabla s_t\rVert_{L_2}^2$ is bounded above and below by positive multiples of $d$, uniformly in $t\in[0,T]$, with constants depending only on the target distribution.
Using these bounds, we construct an explicit family of noise schedules whose parameters can be chosen as functions of problem constants and $n$ to attain the $\tilde{\mathcal O}(d/n)$ sampling error rate.

\begin{theorem}\label{thm:schedules}
Suppose \Cref{ass:fisher_information,ass:surrogate} hold.
For every sufficiently large $n$ and any continuous $f\in\mathcal F$, let $q$ be a continuous function such that $q(t)\geq 2f(t)$ for all $t\in[0,T]$.
Then, for any such $q$ and constants $c,C>0$ independent of $n$ and $d$,
\[
    \sup_{t\in[0,T]}q(t)\leq C\log(1+n),
    \quad
    \int_0^T\exp\left(\int_0^t q(\tilde t)\,d\tilde t\right)dt
    \geq\frac{c}{g_0}\frac{\lVert X_\star\rVert_{L_2}^2}{d}\frac{n}{\log(1+n)},
\]
where $g_0$ satisfies $0 < g_0 \leq C\frac{d}{\fisher_\star}\log(1+n)$.
Then, the schedule $(f,\appx_q)$ defined as
\begin{equation}\label{eq:admissable_g}
    \appx_q(t)
    =
    g_0\exp\left(
    -\int_0^t \bigl(2f(\tilde t)-q(\tilde t)\bigr)\,d\tilde t
    \right)
\end{equation}
achieves $\tilde{\mathcal O}(d/n)$ sampling error in the KL divergence, ignoring problem constants in \Cref{ass:surrogate}.
\end{theorem}

\Cref{thm:sampling_error} establishes a $\tilde{\mathcal O}(d/n)$ sampling error bound under the mild regularity conditions in \Cref{ass:fisher_information}.
\Cref{thm:schedules} imposes additional distributional assumptions to obtain an explicit family of schedules attaining this rate.
As we show below, this family encompasses the functional forms of several schedules used in practice, including VE-exponential, VP-sigmoid, and VP-linear.
Beyond recovering these forms, the theorem provides insights on how their parameters should scale with problem constants and the sampling budget $n$ to obtain the $\tilde{\mathcal O}(d/n)$ KL error guarantee.
These conditions guide both the choice of parameters for familiar schedules and the construction of new schedules with the same guarantee.

\Cref{thm:schedules} gives a family of schedules through choices of \(f\) and \(q\).
We now impose an affine relationship between \(f\) and \(g\) to show how \Cref{thm:schedules} can be used for noise schedule design.
\begin{equation}\label{eq:ACS}
    \Theta_{\textrm{ACS}} = \left\{ (f, g) \in \Theta: f(t) = \theta g(t) + \omega \text{ for some } \theta, \omega \geq 0 \right\}
    .
\end{equation}
We call this class affinely coupled schedules (ACSs).
This ACS family provides a theoretical framework for several empirically successful noise schedules, and identifies conditions under which they attain the $\tilde{\mathcal{O}}(d/n)$ sampling error rate.
\begin{corollary}\label{cor:optimal_schedules_practice}
For a schedule $(f,\appx_q)$ of the form \eqref{eq:admissable_g} in the ACS class \eqref{eq:ACS},
\begin{equation}\label{eq:optimal_schedules_practice}
    \appx_q(t)
    = \frac{g_0\exp\left(\int_0^t\bigl(q(\tilde t)-2\omega\bigr)\,d\tilde t\right)}
    {1+2\theta g_0\int_0^t\exp\left(\int_0^s\bigl(q(\tilde t)-2\omega\bigr)\,d\tilde t\right)ds},
    \qquad f(t)=\theta\appx_q(t)+\omega.
\end{equation}
\end{corollary}
The ACS family recovers the VE-exponential schedule \citep{song2020generativemodelingestimatinggradients} and VP-linear \citep{ho2020denoising}.
It also yields sigmoidal VP schedules similar to those used in \citep{xu2022geodiff,jabri2023scalableadaptivecomputationiterative,chen2023importancenoiseschedulingdiffusion}.
The corresponding choices of $q$, $\theta$, $\omega$, and $g_0$ are:
\begin{itemize}
    \item \textbf{VE-Exponential:}
    Setting $q(t)=\bar q = (2/T) \log (1 + n)$ and $(\theta,\omega)=(0,0)$ gives $\appx_q(t)=g_0e^{\bar q t}$.
    A choice of $g_0$ independent of $n$ and the above $q$ satisfy the conditions of \Cref{thm:schedules} for all sufficiently large $n$.

    \item \textbf{VP-Sigmoid:}
    Setting $q(t)=\bar q = \frac{2}{T}\log(1+n)$, $g_0>0$ independent of $n$, and $(\theta,\omega)=(1/2,0)$ gives
    \begin{equation*}
        \appx_q(t)=\frac{g_0\bar q}{g_0+(\bar q-g_0)e^{-\bar q t}}.
    \end{equation*}
    The conditions of \Cref{thm:schedules} are satisfied for sufficiently large $n$.

    \item \textbf{VP-Linear:}
    For $(\theta,\omega)=(1/2,0)$, $g_{\min} = g_0>0$ independent of $n$, and $ g_{\max}=g_{\min}+\tfrac{2}{T}\log(1+n)$,
    \begin{equation*}
        q(t)=g_{\min}+\frac{g_{\max}-g_{\min}}{T}t
        +\frac{g_{\max}-g_{\min}}{Tg_{\min}+(g_{\max}-g_{\min})t}
    \end{equation*}
    satisfies the conditions of \Cref{thm:schedules} for sufficiently large $n$.
    The corresponding $\appx_q (t)$ is the linear schedule $\appx_q(t)=g_{\min}+(g_{\max}-g_{\min})t/T$.

\end{itemize}
We show in \Cref{sec:connections} that the choice of $q$ in VP-Sigmoid and VP-Linear above satisfies the conditions of \Cref{thm:schedules} for sufficiently large $n$, and therefore achieves the $\tilde{\mathcal{O}}(d/n)$ error.

\Cref{cor:optimal_schedules_practice} goes beyond recovering the functional forms of VP-Sigmoid and VP-Linear.
Standard implementations typically choose schedule parameters independently of $n$, whereas the examples above provide explicit $n$-dependent parameterizations.
Once the $n$-independent constants are chosen, these formulas adapt each schedule to the inference budget without separate retuning.

\section{Proof Sketches for \texorpdfstring{\Cref{thm:ub,thm:sampling_error,thm:schedules}}{the Main Theorems}}\label{sec:analysis}
Here we present proof sketches for \Cref{thm:ub,thm:sampling_error,thm:schedules}.
The full proofs and intermediate lemmas are presented in \Cref{app:main_results,app:schedules}.

\paragraph{Proof Sketch for \Cref{thm:ub}.}
The goal is to upper bound $\mathrm{KL}(p_\star \| \hat{p}_\star)$, the KL-divergence between the data distribution $p_\star$ and the output distribution $\hat{p}_\star$.

\textbf{Step 1: Error Decomposition.}
Following standard techniques that use Girsanov theorem and data processing inequality to analyze SDE samplers (see, e.g., \citep{song2021maximumlikelihoodtrainingscorebased,chen2022sampling,liuLetUsBuild2022}), we can decompose the total sampling error $\mathrm{KL}(p_\star \| \hat{p}_\star)$ into the initialization error and the accumulated discretization error over the reverse SDE steps:
\begin{equation}\label{eq:girsanov}
        \mathrm{KL}(p_\star \| \hat{p}_\star)
        \leq
        \underbrace{\frac{\alpha_T^2}{2 \sigma_T^2} \lVert X_\star \rVert^2_{L_2 (p_\star)}}_{\text{initialization error}}
        +
        \underbrace{\frac{1}{2} \sum_{k=0}^{n-1} \int_{t_k}^{t_{k+1}} g(t) \left\lVert
        s_{t} (X_t) - s_{t_{k+1}} (X_{t_{k+1}})\right\rVert_{L_2}^2 d t}_{\text{discretization error}} .
\end{equation}
Here the discretization error is written in terms of forward process quantities using the distributional equivalence $X_\tau^\leftarrow \overset{d}{=} X_{T-\tau},\tau\in[0,T]$, and $t_k=T-\tau_{n-k}$.
The remainder of the proof focuses on bounding the discretization error.

\textbf{Step 2: Local Discretization Error.}
We now bound $\left\lVert s_{t} (X_t) - s_{t_{k+1}} (X_{t_{k+1}})\right\rVert_{L_2}^2$ for $t\in[t_k,t_{k+1}]$.
For this, it is more convenient to work with $s_{T-\tau} (X_\tau^\leftarrow)$ and show that it satisfies an SDE (\Cref{prop:dst}):
\begin{equation}\label{eq:dst}
    ds_{T-\tau} (X_\tau^\leftarrow)
    =
    - f(T-\tau) s_{T-\tau}(X_\tau^\leftarrow) d\tau
    + \sqrt{g(T-\tau)} \nabla s_{T-\tau} (X_\tau^\leftarrow) dB_\tau.
\end{equation}
When analyzing the discretization error of an SDE within a small time interval of length $h$, the key fact to leverage is the It\^o isometry, which says for a stochastic integral with respect to Brownian motion, $\mathbb{E}[(\int Y_t dB_t)^2]=\int \mathbb{E}[Y_t^2]dt.$
For sufficiently small $h$, the discretization error due to the drift term is smaller than the error arising from the diffusion term.
Integrating and using $g \in \mathcal G$ gives 
\begin{equation}\label{eq:local_discretization}
    \begin{aligned}
        \int_{\tau_k}^{\tau_{k+1}}g(T-\tau)
        \left\lVert s_{T-\tau}(X_\tau^\leftarrow)-s_{T-\tau_k}(X_{\tau_k}^\leftarrow)\right\rVert_{L_2}^2\,d\tau
        \lesssim
        h\int_{\tau_k}^{\tau_{k+1}}g(T-\tau)^2
        \lVert\nabla s_{T-\tau}(X_\tau^\leftarrow)\rVert_{L_2}^2\,d\tau.
    \end{aligned}
\end{equation}

\textbf{Step 3: Fisher Information ODE.}
Now a key term to tackle is $\lVert \nabla s_t (X_t) \rVert^2_{L_2}$ in the bound \eqref{eq:local_discretization}.
Using Bochner's formula, we show that it is related to the Fisher information via the following ODE (\Cref{lem:fisher_ode}):
\begin{equation}\label{eq:fisher-equiv}
    \dot{\fisher}_t = 2f(t)\fisher_t - g(t) \lVert \nabla s_t(X_t) \rVert_{L_2}^2 \quad \Leftrightarrow \quad g(t) \lVert \nabla s_t(X_t) \rVert_{L_2}^2 = 2f(t)\fisher_t-\dot{\fisher}_t.
\end{equation}

\textbf{Step 4: Final Error Bound.}
Substituting \eqref{eq:fisher-equiv} into the discretization error gives
\begin{equation}
    \text{discretization error} \lesssim h \int_0^T g(t)(2f(t)\fisher_t-\dot{\fisher}_t)dt.
\end{equation}
However, this is still not easy to optimize since it involves $f(t),g(t),J_t$, which are related.
Therefore, we eliminate $g(t)$ by writing the bound as
\begin{equation}
    h\int_0^T g(t)(2f(t)\fisher_t-\dot{\fisher}_t)dt
    =h\int_0^T \frac{(2f(t)\fisher_t-\dot{\fisher}_t)^2}{\lVert \nabla s_t(X_t) \rVert_{L_2}^2}dt,
\end{equation}
and using the concavity of entropy power $\fisher_t^2\le d \lVert \nabla s_t(X_t) \rVert_{L_2}^2$ (\Cref{lem:st_Fisher_lb}), which gives the final error bound
\begin{equation}\label{eq:discretization-upper}
    \text{discretization error} \lesssim h d \int_0^T \left( 2 f(t) - \frac{\dot{\fisher}_t}{\fisher_t} \right)^2 dt.
\end{equation}

\paragraph{Proof Sketch for \Cref{thm:sampling_error}.}
For any $f$, we aim to find the optimal $g^*$, which must produce a trajectory $(\fisher_t^* : t \in [0, T])$ that minimizes the upper bound on the discretization error in \eqref{eq:discretization-upper}.
This problem can be solved via the Euler-Lagrange equation, which states that to minimize a functional $\int_0^T \mathcal{L}(t, u_t, \dot{u}_t) dt$, the optimal condition is described by
\begin{equation}\label{eq:EL}
    \frac{\partial \mathcal{L}}{\partial u_t} - \frac{d}{dt} \frac{\partial \mathcal{L}}{\partial \dot{u}_t} = 0 .
\end{equation}
Let $u_t=\log \fisher_t$, so $\dot{u}_t=\dot{\fisher}_t/\fisher_t$ and $\mathcal{L}(t, u_t, \dot{u}_t)=(2f(t)-\dot{u}_t)^2$.
Then the optimal condition for the optimal trajectory $(\fisher_t^*: t \in [0, T])$ is
\begin{equation}\label{eq:optimal_fisher}
    2 f(t) - \frac{\dot{\fisher}_t^*}{\fisher_t^*} = \lambda,
\end{equation}
where $\lambda>0$ is a time-independent constant.
This condition also gives the characterization of the optimal $g^*$ in \eqref{eq:gstar_form}.
Then the discretization error bound becomes
\begin{equation*}
    \text{discretization error}\lesssim hd\lambda^2T.
\end{equation*}
Balancing the rates of the initialization and discretization errors yields the $\mathrm{KL}(p_\star \| \hat{p}_\star)=\tilde{\mathcal{O}}(d/n)$ bound.

\paragraph{Proof Sketch for \Cref{thm:schedules}.}
We show that the choice of $\appx_q (t)$ in \eqref{eq:admissable_g} implies the $\tilde{\mathcal{O}}(d/n)$ rate under \Cref{ass:surrogate}.
The key implication of \Cref{ass:surrogate} is that lower and upper bounds on $\fisher_t^2/\| \nabla s_t \|_{L_2}^2$ can be established: for some problem constant $\kappa$, 
\[
    \frac{d}{\kappa^2}
    \leq
    \frac{\fisher_t^2}{\lVert\nabla s_t\rVert_{L_2}^2}
    \leq d
    .
\]
The schedule $\appx_q (t)$ was constructed to satisfy $\tfrac{\dot{\appx}_q(t)}{\appx_q(t)}=q(t)-2f(t)$ so that the ODE \eqref{eq:fisher-equiv} yields
\[
    \frac{d}{dt}\bigl(\appx_q(t)\fisher_t\bigr)
    \leq
    q(t)\appx_q(t)\fisher_t
    -\frac{\bigl(\appx_q(t)\fisher_t\bigr)^2}{d}.
\]
By the choice of $g_0$ in \Cref{thm:schedules}, $g_0\fisher_\star\leq Cd\log(1+n)$.
Together with the condition $\sup_{t\in[0,T]}q(t)\leq C\log(1+n)$, this gives
\[
    \appx_q(t)\fisher_t\leq Cd\log(1+n),
    \qquad t\in[0,T].
\]
Moreover, $q(t)\geq2f(t)$ implies that $\appx_q$ is nondecreasing.
Following the proof of \Cref{thm:ub}, we use these properties to obtain that the discretization error is bounded by a multiple of
\[
    \frac{d\log^2(1+n)}{n}
    +
    \frac{d\log^3(1+n)}{n^2}.
\]
The second condition on $q$ ensures that the initialization error is $\tilde{\mathcal{O}}(d/n)$, and we conclude $\mathrm{KL}(p_\star \| \hat{p}_\star) = \tilde{\mathcal{O}}(d/n)$. 

\section{Experiments}\label{sec:experiments}
Our theoretical analyses show that ACS encompasses several existing noise schedules and identify conditions under which these schedules achieve the $\mathrm{KL}(p_\star \| \hat p_\star) = \tilde{\mathcal{O}}(d/n)$ sampling error rate.
This does not imply equivalent finite-budget performance: different schedules induce different Fisher-information trajectories and may differ in the constants and tradeoffs suppressed by the asymptotic bound, leaving open how different members of the ACS family compare on real-world data.

We describe our choice of parametrization for the noise schedule $(f_n, g_n)$ in \Cref{thm:schedules}.
For some time-varying function $\psi$, let $q_n (t) = \lambda_n \psi(t)$.
Following the proof of \Cref{thm:schedules} in \Cref{proof:admissable_g_robustness}, the bound on the sampling error can be written approximately as
\begin{equation*}
    C_1 e^{-\lambda_n}
    + C_2 \frac{\lambda_n^2}{n} ,
\end{equation*}
where $C_1, C_2$ are problem constants.
For any $n$, the minimizer of this approximate upper bound is given by $\lambda_n = \mathcal{W}(K n)$ for some constant $K$ independent of $n$, where $\mathcal{W}$ is the Lambert-$W$ function.
This choice of $\lambda_n$ balances the initialization error proportional to $e^{-\lambda_n}$ and the discretization error $\sim \lambda_n^2 /n $.
Therefore, we choose $(f_n, g_n)$ to vary with $n$ only through their dependence on $\lambda_n$.
The exact parameterization is given in \Cref{app:experiments}.

We evaluate all schedules using identical pretrained checkpoints from \citep{karrasElucidatingDesignSpace2022,karras2024analyzing}, without fine-tuning, corresponding to each dataset and the same ancestral sampler.
``Power'' uses the published \citep{karrasElucidatingDesignSpace2022} EDM power law noise grid ($\sigma_{\min}=0.002, \sigma_{\max}=80, \rho=7$).
The ACS and all other schedule hyperparameters were tuned once to minimize the mean log-FID scores at NFE 10 and 50 with 10,000 samples on CIFAR-10 and ImageNet $64 \times 64$, then held fixed for all NFEs and datasets.
Performance on these held-out datasets is reported in \Cref{fig:fid50k}, and results on the tuning datasets are reported in \Cref{fig:fid50k_selection}.

\begin{figure}[th]
    \centering
    \includegraphics[width=\linewidth]{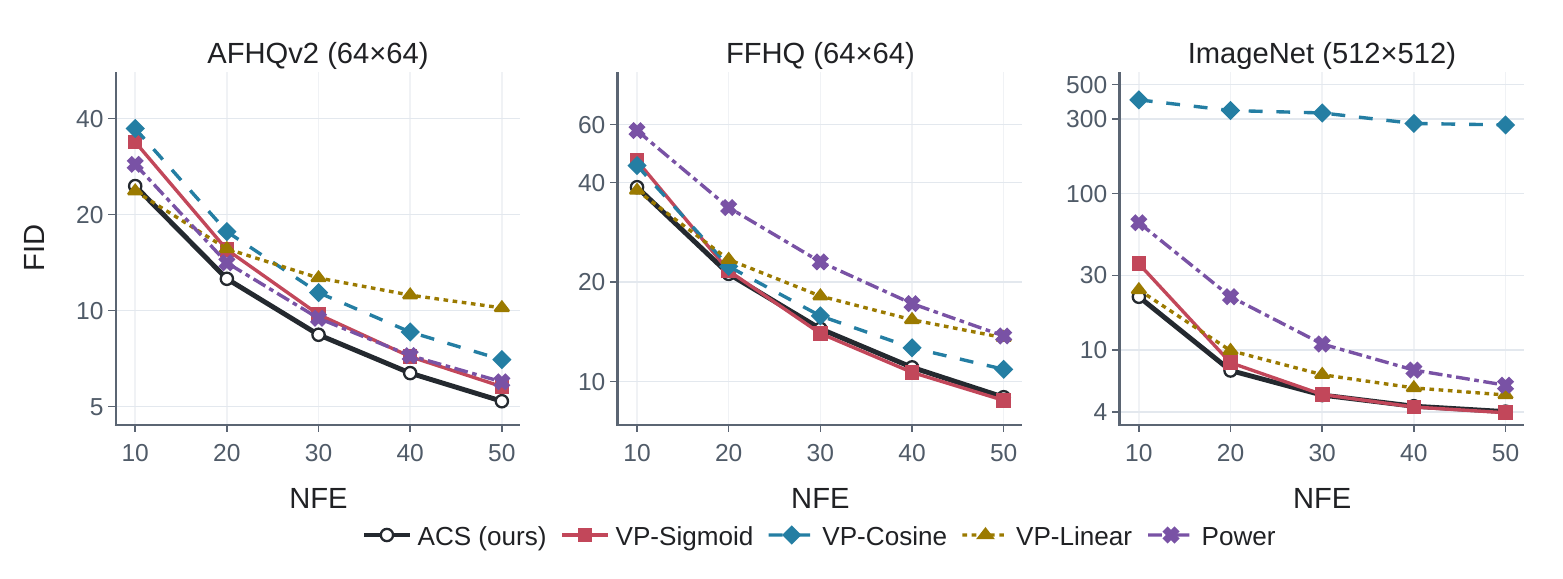}
    \caption{FID on 50,000 generated samples across inference budgets on datasets not used for hyperparameter selection.
    VP-Linear yields substantially lower FID than VP-Sigmoid at NFE 10, whereas VP-Sigmoid generally yields lower FID at higher NFE budgets.
    ACS balances these two regimes and remains competitive across the full range of inference budgets.
    The cosine schedule performed well on its tuning datasets and the held-out datasets AFHQv2 and FFHQ, but transferred poorly to ImageNet $512 \times 512$.
    }\label{fig:fid50k}
\end{figure}

\Cref{fig:fid50k} shows that the relative performance of existing schedules depends on the inference budget: VP-Linear and Power are competitive at low NFEs, whereas VP-Sigmoid becomes more competitive as NFE increases.
ACS remains competitive across the full NFE range.
These results show that the $n$-dependent ACS construction balances the inference budget tradeoffs exhibited by standard schedules.

\section{Conclusion}\label{sec:conclusion}
We introduced an optimal control framework for designing noise schedules, utilizing a Fisher information ODE to bypass intractable PDE constraints.
Our approach guarantees $\tilde{\mathcal{O}}(d/n)$ bounds on the KL divergence and yields a novel class of Affinely Coupled Schedules (ACS) whose $n$-dependent parameterization achieves low FID across image-generation tasks and inference step counts.
Our sampling error guarantee holds under mild regularity conditions, and the additional SLC or GMM assumptions are used to obtain explicit schedule constructions.
Future work will seek similarly explicit constructions under weaker distributional assumptions.

\authorstatements
\bibliographystyle{iclr2027_conference}
\bibliography{bibliography}

\newpage
\appendix

\section{Additional Related Works}\label{app:related_works}

Early theoretical analyses of diffusion models \citep{blockGenerativeModelingDenoising2022,bortoliConvergenceDenoisingDiffusion2023,chen2022sampling} focused predominantly on SDE samplers utilizing first-order discretization schemes under constant noise schedules.
A primary objective in this line of work has been to obtain a bound on the sampling error, typically measured in the total variation (TV) distance.
Because their analysis is restricted to constant noise schedules, their bounds scale sub-optimally with respect to the data dimension $d$ and problem constants.
\citet{hchen_improved_score,conforti2024klconvergenceguaranteesscore} strengthened these results by establishing bounds on the KL divergence, which implies a similar scaling in the TV distance through Pinsker's inequality.

The effect of time-varying noise schedules has been explored more recently.
\citet{benton2024nearly,conforti2024klconvergenceguaranteesscore} showed that by utilizing time-varying noise schedules, the KL error bound's dependence on the dimension $d$ and problem constants can be improved significantly such that it grows linearly with $d$.
Their results imply that the iteration complexity $n \asymp d/\epsilon^2$ is sufficient to achieve $\mathrm{KL}(p_\star \| \hat{p}_\star) \lesssim \epsilon^2$ error bounds.
These improved error bounds are achieved by considering constant $(f, g)$ and time-varying discretization step size $\{h_k\}$.
As mentioned in \citep[Remark 3]{hchen_improved_score}, such schedules can be viewed as fixing a constant-in-time $f$ and using $g$ that increases exponentially with time.
While these works establish $\tilde{\mathcal{O}}(d/n)$ rates, the specific schedule constructions required by their analyses depart from the standard affinely coupled schedules commonly favored in empirical implementations.

Beyond tuning the standard time-reversal schedule, a conceptually related problem explored by \citet{chen2025lipschitz} seeks to design a reverse process that differs from the exact time-reversal of the forward process.
To design this mismatched schedule, \citet{chen2025lipschitz} propose to minimize the Lipschitz constant of the resulting reverse drift.
However, such generation processes also depart from standard empirical pipelines.
In contrast, our framework aligns with practice by maintaining the exact time-reversal of the true forward SDE and directly minimizes a rigorous upper bound on the KL sampling error, rather than relying on a regularity proxy.

Lastly, noise schedule design was considered in the context of Probability Flow ODEs.
\citet{gao2024convergence} analyze the Wasserstein-2 error, requiring the data distribution be strongly log-concave.
Under this regime, they conclude that VE schedules exhibit worse iteration complexity than VP schedules.
Although ODE samplers can be effective at low NFE, stochastic samplers can produce higher-quality samples \citep{ashiq2026ddimhallucinatesddpmtheoretical}.
Our work analyzes stochastic SDE samplers using the KL divergence and propose an optimal control framework based on Fisher information dynamics.
We require only mild regularity conditions to prove $\tilde{\mathcal{O}}(d/n)$ bounds, where a strongly log-concave or a Gaussian mixture model is assumed only to derive closed-form parametrizations.
Under this setting, we prove that both several existing schedules for SDE samplers can achieve the $\tilde{\mathcal{O}}(d/n)$ error rate if their parameters are chosen to scale with the number of discretization steps $n$ appropriately.

\section{Supplement to \texorpdfstring{\Cref{sec:preliminaries}}{Problem Setup}}\label{app:preliminaries}
\subsection{Statistical Equivalence of Score Functions}\label{app:score_function_statistical_equivalence}
We describe how to map a score function trained for one noise schedule to the score function corresponding to another noise schedule using their statistical equivalence as discussed in \citep{chen2025lipschitz}.
In this section, we use $\tau_t$ to be a map from one time-scale $t$ to another time-scale $\tau_t$.

Let $X_t^{(1)} = \alpha_t^{(1)} X_\star + \sigma_t^{(1)} Z$ and $X_t^{(2)} = \alpha_t^{(2)} X_\star + \sigma_t^{(2)} Z$ be two random variables parameterized by $(\alpha_t^{(i)}, \sigma_t^{(i)})$.
We first identify the equivalent time $\tau_t$ in the first schedule that matches the SNR of the second schedule at time $t$:
\begin{align*}
    \frac{\alpha_{\tau_t}^{(1)}}{\sigma_{\tau_t}^{(1)}} = \frac{\alpha_t^{(2)}}{\sigma_t^{(2)}}
\end{align*}
Because the SNRs are identical, $X_t^{(2)}$ and $X_{\tau_t}^{(1)}$ differ strictly by a scaling factor.
Specifically,
\begin{align*}
    X_t^{(2)}
    &=
    \sigma_t^{(2)} \left(\frac{\alpha_t^{(2)}}{\sigma_t^{(2)}} X_\star + Z \right)
    =
    \sigma_t^{(2)} \left(\frac{\alpha_{\tau_t}^{(1)}}{\sigma_{\tau_t}^{(1)}} X_\star + Z \right)
    =
    \frac{\sigma_t^{(2)}}{\sigma_{\tau_t}^{(1)}}
    X_{\tau_t}^{(1)}
    .
\end{align*}
Applying the change of variables formula for probability densities, $p_t^{(2)}(x) = \left( \frac{\sigma_{\tau_t}^{(1)}}{\sigma_t^{(2)}} \right)^d p_{\tau_t}^{(1)}\left( \frac{\sigma_{\tau_t}^{(1)}}{\sigma_t^{(2)}} x \right)$. Taking the gradient of the logarithm of this density with respect to $x$ yields the translation:
\begin{equation}\label{eq:score_equivalence}
    s_t^{(2)}(x) = \frac{\sigma_{\tau_t}^{(1)}}{\sigma_t^{(2)}} s_{\tau_t}^{(1)}\left( \frac{\sigma_{\tau_t}^{(1)}}{\sigma_t^{(2)}} x \right)
\end{equation}

\textbf{Score Function from Noise Prediction Model:}
It is common to train a noise prediction model
\begin{align*}
    S_t (x) = \mathbb{E}[ Z \mid X_t = x]
\end{align*}
in place of the score function $s_t$.
The score of the conditional distribution of $(X_t \mid X_\star) \sim \mathcal{N}(\alpha_t X_\star, \sigma_t^2 I)$ is given by
\begin{align*}
    \nabla_{x_t} \log p(x_t \mid x_\star) = -\frac{x_t - \alpha_t x_\star}{\sigma_t^2} = -\frac{Z}{\sigma_t} .
\end{align*}
The marginal score is the expectation of the conditional score, which gives the general relationship
\begin{align*}
    s_t (x) = -\frac{1}{\sigma_t} S_t(x)
\end{align*}

\subsection{Non-Zero Score Matching Error}\label{app:score_matching}
For any noise schedule $(f, g)$ and reverse-time grid $\{\tau_k\}_{k=0}^n$, define the score matching error $\epsilon_{sc}$ as
\begin{align*}
    \epsilon_{sc}^2 \coloneqq
    \sum_{k=0}^{n-1}
    \left(\int_{\tau_k}^{\tau_{k+1}} g(T-\tau)\,d\tau\right)
    \left\lVert s_{T-\tau_k}-\hat{s}_{T-\tau_k}\right\rVert_{L_2(p_{T-\tau_k})}^2
    .
\end{align*}
Using that
\[
    s_{T-\tau}(X_\tau^\leftarrow)-\hat s_{T-\tau_k}(X_{\tau_k}^\leftarrow)
    =s_{T-\tau}(X_\tau^\leftarrow)-s_{T-\tau_k}(X_{\tau_k}^\leftarrow)
    +s_{T-\tau_k}(X_{\tau_k}^\leftarrow)-\hat s_{T-\tau_k}(X_{\tau_k}^\leftarrow) ,
\]
and the Girsanov theorem and the data processing inequality as in \citep{chen2022sampling,benton2024nearly,conforti2024klconvergenceguaranteesscore}, the bound \eqref{eq:girsanov} is replaced by
\begin{equation}\label{eq:girsanov_sm}
    \mathrm{KL}(p_\star \| \hat{p}_\star)
    \leq
    \epsilon_{sc}^2
    + \frac{\alpha_T^2}{2 \sigma_T^2} \lVert X_\star \rVert_{L_2}^2
    +
    \sum_{k=0}^{n-1} \int_{\tau_k}^{\tau_{k+1}} g(T-\tau) \left\lVert
    s_{T-\tau}(X_\tau^\leftarrow)-s_{T-\tau_k}(X_{\tau_k}^\leftarrow)\right\rVert_{L_2}^2 d\tau .
\end{equation}
To address the impact of the score matching error $\epsilon_{sc}^2$, we evaluate two distinct paradigms: utilizing an existing pre-trained score network, and training a new model from scratch.

\textbf{Pre-Trained Statistical Equivalence:}
Let $(f_1, g_1)$ and $(f_2, g_2)$ denote two distinct noise schedules with corresponding marginal distributions characterized by parameters $(\alpha_t^{(1)}, \sigma_t^{(1)})$ and $(\alpha_t^{(2)}, \sigma_t^{(2)})$. Suppose there exists a monotonically increasing time-map $\tau_t$ such that the Signal-to-Noise Ratios (SNRs) of the two schedules are strictly equivalent, i.e.,
\begin{equation}
    \frac{\alpha_{\tau_t}^{(1)}}{\sigma_{\tau_t}^{(1)}} = \frac{\alpha_t^{(2)}}{\sigma_t^{(2)}}, \quad \forall t \in [0, T_2].
\end{equation}
As established in \Cref{app:score_function_statistical_equivalence}, under this SNR equivalence, the random variables $X_t^{(2)}$ and $X_{\tau_t}^{(1)}$ differ only by a deterministic spatial scaling.
Consequently, the score network $\hat{s}^{(2)}$ corresponding to schedule $(f_2, g_2)$ can be derived from a pre-trained score function $\hat{s}^{(1)}$ as in \eqref{eq:score_equivalence}.
Applying this change of variables to both the exact and approximate scores gives
\begin{equation*}
    \left\lVert s_t^{(2)}-\hat{s}_t^{(2)}\right\rVert_{L_2(p_t^{(2)})}^2
    = \left(\frac{\sigma_{\tau_t}^{(1)}}{\sigma_t^{(2)}}\right)^2
    \left\lVert s_{\tau_t}^{(1)}-\hat{s}_{\tau_t}^{(1)}\right\rVert_{L_2(p_{\tau_t}^{(1)})}^2.
\end{equation*}
Thus, $\epsilon_{sc}^2$ for $(f_2,g_2)$ is controlled by the pre-trained model's errors at the matched grid times, with the spatial scaling factors above and the weights in its definition.

\textbf{Universal Approximation:}
Alternatively, when training a score network $\hat{s}_t(x)$ from scratch under noise schedule $(f, g) \in \Theta$, the score matching error can be driven arbitrarily close to zero.
To formalize this, we rewrite the score matching objective as an $L_2$ distance with respect to a joint space-time measure.
Define the measure on $\mathbb{R}^d \times [0, T]$ by
\begin{equation*}
    d\mu(x,t)
    = \sum_{k=0}^{n-1}
    \left(\int_{\tau_k}^{\tau_{k+1}}g(T-\tau)\,d\tau\right)
    p_{T-\tau_k}(x)\,dx\,\delta_{T-\tau_k}(dt),
\end{equation*}
where $\delta_{T-\tau_k}$ is the Dirac measure at $T-\tau_k$.
The score matching error is then equivalently given by:
\begin{equation}
    \epsilon_{sc}^2 = \int_0^T \int_{\mathbb{R}^d} \|\hat{s}_t(x) - s_t(x)\|^2 d\mu(x, t).
\end{equation}
Because $p_t(x)$ integrates to $1$ over $\mathbb{R}^d$, the total measure is $\mu(\mathbb{R}^d \times [0, T]) = \int_0^T g(t) dt$.
For any schedule $(f, g) \in \Theta$, this integral is finite, making $\mu$ a finite measure.
Furthermore, the target score function $s_t(x)$ belongs to $L_2(\mu)$, since
\begin{equation}
    \| s_t \|_{L_2(\mu)}^2
    = \sum_{k=0}^{n-1}
    \left(\int_{\tau_k}^{\tau_{k+1}}g(T-\tau)\,d\tau\right)
    \mathcal{J}_{T-\tau_k} < \infty.
\end{equation}
The finiteness of this sum follows from the integrability of $g(t)$ and \Cref{ass:fisher_information}.
Under these conditions, the universal approximation theorem (e.g., \citep{hornik1991approximation}) guarantees that the class of neural networks is dense in $L_2(\mu)$.
Therefore, for each fixed schedule and grid, there exists a network parameterization with arbitrarily small $\epsilon_{sc}^2$ given sufficient model capacity.

\subsection{Constant Step Size Suffices for Time-Varying Noise Schedules}\label{app:constant_step_size}
We show that choosing $(f, g)$ with constant step size is without loss of generality.
Consider a choice of noise schedule $(f_1, g_1)$ and time-varying step size with time grid $\{\tau_k^{(1)}\}$.
Denote by $\{\hat{X}_{\tau_k^{(1)}}^{\leftarrow, (1)}\}$ the EI discretized process according to this noise schedule and time grid.
We show that there exists a noise schedule $(f_2, g_2)$ and time grid $\{\tau_k^{(2)}\}$ generated by constant step size $h$, i.e., $\tau_k^{(2)} = k h$, such that the EI discretized process $\{\hat{X}_{\tau_k^{(2)}}^{\leftarrow, (2)}\}$ admits the same distribution as $\{\hat{X}_{\tau_k^{(1)}}^{\leftarrow, (1)}\}$.

Fix $(f_1, g_1)$, $\{\tau_k^{(1)}\}$, and $\{\tau_k^{(2)}\}$ where $\tau_k^{(2)} = k h$.
Let $\phi:[0,T]\to[0,T]$ be a continuously differentiable map with $\phi'(\tau)>0$ for all $\tau\in[0,T]$ and $\phi(\tau_k^{(2)})=\tau_k^{(1)}$ for all $k$.
We consider the construction
\begin{equation}\label{eq:schedule_tv_ti_map}
    f_2 (T - \tau^{(2)}) = f_1 (T - \phi(\tau^{(2)})) \cdot \phi'(\tau^{(2)}),
    \quad
    g_2 (T - \tau^{(2)}) = g_1 (T - \phi(\tau^{(2)})) \cdot \phi'(\tau^{(2)}) ,
\end{equation}
for all $\tau^{(2)}\in[0,T]$.
The SDE governing $\{\hat{X}_{\tau_k^{(1)}}^{\leftarrow, (1)}\}$ is
\begin{equation}\label{eq:EI_tv_step}
    d\hat{X}_{\tau^{(1)}}^{\leftarrow, (1)}
        =
    \left[ f_1(T - \tau^{(1)}) X_{\tau^{(1)}}^{(1)} + g_1(T - \tau^{(1)}) s_{T-\tau_k^{(1)}}(X_{\tau_k^{(1)}}^{(1)}) \right] d\tau^{(1)} + \sqrt{g_1(T - \tau^{(1)})} dB_{\tau^{(1)}}
\end{equation}
Substituting the differentials
\begin{align*}
    d\tau^{(1)} = \phi'(\tau^{(2)}) d\tau^{(2)},
    \quad
    dB_{\phi(\tau^{(2)})} = \sqrt{\phi'(\tau^{(2)})} dB_{\tau^{(2)}}
\end{align*}
for \eqref{eq:EI_tv_step}, we obtain
\begin{align*}
    d\hat{X}_{\tau^{(2)}}^{\leftarrow, (2)}
        &=
    \Big[ f_1(T - \phi(\tau^{(2)})) \hat{X}_{\tau^{(2)}}^{\leftarrow, (2)}
    \\ &
    + g_1(T - \phi(\tau^{(2)})) s_{T-\tau_k^{(1)}}(\hat{X}_{\tau_k^{(2)}}^{(2)}) \Big] \phi'(\tau^{(2)}) d\tau^{(2)} + \sqrt{g_1(T - \phi(\tau^{(2)})) \phi'(\tau^{(2)})} dB_{\tau^{(2)}}
    .
\end{align*}
Substituting \eqref{eq:schedule_tv_ti_map} for the above equation, we obtain
\begin{align*}
    d\hat{X}_{\tau^{(2)}}^{\leftarrow, (2)}
    &=
    \left[ f_2(T - \tau^{(2)}) \hat{X}_{\tau^{(2)}}^{\leftarrow, (2)} + g_2(T - \tau^{(2)}) s_{T-\phi(\tau_k^{(2)})}(X_{\tau_k^{(2)}}^{\leftarrow, (2)}) \right] d\tau^{(2)}
    \\ &
    + \sqrt{g_2(T - \tau^{(2)})} dB_{\tau^{(2)}}
    .
\end{align*}
Since $\hat{X}_{\tau^{(2)}}^{\leftarrow, (2)} \overset{d}{=} \hat{X}_{\tau^{(1)}}^{\leftarrow, (1)}$, the sequence of variables at the respective grid points $\{\tau_k^{(1)}\}$ and $\{\tau_k^{(2)}\}$ have the same distribution.

\section{Preliminaries}

Throughout the paper, we use the notation $(\alpha_t, \sigma_t)$ to denote the coefficients of $X_t = \alpha_t X_\star + \sigma_t Z$, which is the solution to the forward SDE~\eqref{eq:dXt}.
The coefficients can be evaluated as $\alpha_t = \exp(-\int_0^t f(\tilde{t}) d\tilde{t})$ and $\sigma_t^2$ as a solution to the ODE
\begin{equation}\label{eq:dsigmat}
    \frac{d}{dt} \sigma_t^2 = - 2f(t) \sigma_t^2 + g(t), \quad \sigma_0 = 0 .
\end{equation}
The name ``variance preserving'' is used to refer to schedules satisfying $f(t) > 0$ and $g(t) = 2f(t)$ because they preserve the variance when $\mathrm{Var}(X_\star) \coloneqq \lVert X_\star \rVert_{L_2}^2 = \lVert Z \rVert_{L_2}^2$.

\textbf{Notations:}
For the proofs in the appendix, we use $t$ and $\tau = T - t$ to denote the corresponding forward and reverse time indices.
The notation $\otimes$ is used to denote the tensor product, and $\lVert \cdot \rVert$ is used to denote the Hilbert-Schmidt norm for order $r$ tensors in $(\mathbb{R}^d)^{\otimes r}$.
The divergence of a vector-valued function $\phi$ is denoted by $\nabla \cdot \phi$.

\subsection{Noise Schedules}\label{app:noise_schedule}
In this section, we describe the coefficients $(f, g)$ that correspond to various noise schedules used in practice.

\textbf{Exponential:}
Variance exploding models were used in \citep{song2020generativemodelingestimatinggradients}, where the corresponding SDE form \citep{song2021scorebasedgenerativemodelingstochastic} is given by
\begin{align*}
    dX_t = \sqrt{g(t)} dB_t,
\end{align*}
where $g(t)$ is chosen to grow exponentially with $t$.

\textbf{Sigmoid:}
We consider the sigmoid schedule of \citet{chen2023importancenoiseschedulingdiffusion}.
Let $\sigma(x)=1/(1+e^{-x})$ and $z(t)=(u_-+(u_+-u_-)t/T)/\eta_{\mathrm{sig}}$, where $u_-<u_+$ and $\eta_{\mathrm{sig}}>0$.
Excluding the clipping, the schedule uses
\[
    \alpha_t^2=\frac{\sigma(u_+/\eta_{\mathrm{sig}})-\sigma(z(t))}{\sigma(u_+/\eta_{\mathrm{sig}})-\sigma(u_-/\eta_{\mathrm{sig}})}.
\]
The corresponding VP coefficients satisfy $f(t)=g_{\mathrm{sigmoid}}(t)/2$, where
\begin{equation}\label{eq:sigmoid}
    g_{\mathrm{sigmoid}}(t)
    =-\frac{d}{dt}\log\alpha_t^2
    =\frac{u_+-u_-}{T\eta_{\mathrm{sig}}}\,
    \frac{\sigma(z(t))(1-\sigma(z(t)))}{\sigma(u_+/\eta_{\mathrm{sig}})-\sigma(z(t))}.
\end{equation}
For values of $t$ such that $1-\sigma(u_+/\eta_{\mathrm{sig}})$ is small relative to $1-\sigma(z(t))$, this rate is approximately $\frac{u_+-u_-}{T\eta_{\mathrm{sig}}}\sigma(z(t))$.

\textbf{Linear:}
The linear schedule, proposed in \citep{ho2020denoising}, parametrizes the sample path solution
\begin{equation}\label{eq:baralpha}
    X_t = \sqrt{\bar{\alpha}_t} X_\star + \sqrt{1 - \bar{\alpha}_t} Z ,
    \quad
    \bar{\alpha}_{t_k} = \prod_{j=0}^{k-1} (1 - \beta_{t_j}) .
\end{equation}
The sequence $\{\beta_{\tilde{t}}\}$ is a linear function $\beta_t = \beta_{\min} + (\beta_{\max} - \beta_{\min}) \cdot (t/T)$.
As shown in \citep{song2021scorebasedgenerativemodelingstochastic}, the expression for $f$ in the limit $n \to \infty$ is
\begin{align*}
    f_{lin}(t) \coloneqq \frac{1}{2} \left( \beta_{\min} + (\beta_{\max} - \beta_{\min}) \frac{t}{T} \right) .
\end{align*}
The parameters are set with $\beta_{\min} = 0.1/T, \beta_{\max} = 20/T$.
While $(f_{lin}, g_{lin})$ are obtained in the limit $n \to \infty$, we refer to the pair as linear schedules.

\textbf{Cosine:}
The cosine schedule \citep{nichol2021improveddenoisingdiffusionprobabilistic} refers to the choice
\begin{align*}
    \sqrt{\bar{\alpha}_t} = \frac{\cos \left(\frac{t/T + s}{1+s} \cdot \frac{\pi}{2} \right)}{\cos \left(\frac{s}{1+s} \cdot \frac{\pi}{2} \right)}, \qquad s = 0.008 .
\end{align*}
Let $\theta(t) = (t/T + s) \pi / (2(1 + s))$ so that $\sqrt{\bar{\alpha}_t} = \cos \theta(t) / \cos \theta(0)$.
In the SDE \eqref{eq:dXt}, this is equivalently written as the VP constraint $f = g/2$ with $    \bar{\alpha}_t = \exp \left(- 2 \int_0^t f(\tilde{t}) d\tilde{t} \right)$, which yields
\begin{align*}
    f(t) = -\frac{d}{dt} \log \sqrt{\bar{\alpha}_t}
    &=
    - \frac{d}{dt} \left(\log \cos \theta (t) - \log \cos \theta (0) \right)
    \\ &
    = \frac{\sin \theta (t)}{\cos \theta (t)}  \theta'(t)
    \\ &
    = \frac{\pi}{2T (1 + s)} \tan \left(\frac{t/T + s}{1+s} \cdot \frac{\pi}{2} \right)
    .
\end{align*}
As a result, the cosine schedule corresponds to $g(t) = (\pi / T(1 + s)) \tan ( (t/T + s)/(1+s) \cdot (\pi / 2))$.

\subsection{Generalized Blachman-Stam Inequality}\label{proof:generalized_stam}

Recall the Blachman-Stam inequality \citep{stamInequalitiesSatisfiedQuantities1959,blachmanConvolutionInequalityEntropy1965} which states that for $X_t = \alpha_t X_\star + \sigma_t Z$,
\begin{align*}
    \frac{1}{\fisher(\alpha_t X_\star + \sigma_t Z)}
    \geq
    \frac{1}{\fisher(\alpha_t X_\star)} + \frac{1}{\fisher(\sigma_t Z)}
    = \frac{\alpha_t^2}{\fisher_\star} + \frac{\sigma_t^2}{d} .
\end{align*}
Denote by $\mathrm{PS}(a, b)$ for two positive scalars $a, b$ the parallel sum $(a^{-1} + b^{-1})^{-1}$.
Rearranging and using the notation $\fisher_t = \fisher(\alpha_t X_\star + \sigma_t Z)$, we obtain the Blachman-Stam inequality
\begin{align*}
    \fisher_t
    \leq
    \mathrm{PS}\left(\frac{\fisher_\star}{\alpha_t^2}, \frac{d}{\sigma_t^2} \right)
    .
\end{align*}
We prove a generalization of Stam's inequality that holds for higher order derivatives.
\begin{lemma}\label{lem:generalized_stam}
    Let $p_t$ and $p_{t|\star}$ be the density and (Gaussian) transition kernels of $X_t = \alpha_t X_\star + \sigma_t Z$ with $X_\star \perp Z$.
    Suppose $m\geq 1$, $t\in(0,T]$, and $\nabla^m p_\star \in L_1 (\mathrm{Leb}^d)$.
    Then,
    \begin{equation}\label{eq:generalized_stam}
        \left\lVert \frac{\nabla^m p_t}{p_t}\right\rVert_{L_2}^2
        \leq
        \mathrm{PS}\left(\alpha_t^{-2m} \left\lVert \frac{\nabla^m p_\star}{p_\star}\right\rVert_{L_2}^2, \frac{\lVert H_m \rVert_{L_2(\gamma^d)}^2}{\sigma_t^{2m}} \right) .
    \end{equation}
\end{lemma}
Importantly, this bound interpolates the squared $L_2$ norm of $\nabla^m p_\star / p_\star$ and that of the limit $\alpha_t/\sigma_t \to 0$.
\begin{proof}
First, we show that \Cref{ass:higher_order} implies that the boundary condition vanishes sufficiently fast for the integration by parts formula
\begin{equation}\label{eq:vanishing_boundary}
\int
\nabla_{x_\star}^m p_{t | \star} (x | x_\star) p_\star (x_\star)  dx_\star
= (-1)^j \int \nabla^{m-j}p_{t | \star} (x | x_\star) \otimes \nabla^j p_\star (x_\star)  dx_\star
, \quad \forall j \in \{1, \cdots, m\}
\end{equation}
to hold.
Denote by $\phi$ the Gaussian transition kernel.
The case for $j = 1$ simplifies to
\begin{align*}
    \int (\nabla^2 \phi) p_\star dx = - \int \nabla \phi \otimes \nabla p_\star dx
    \Leftrightarrow
    \int \nabla (p_\star \nabla \phi) dx_\star = 0
    .
\end{align*}
The second identity above holds if $p_\star \nabla \phi$ and $\nabla  (p_\star \nabla \phi)$ are in $L_1$.
Since $\nabla \phi$ is bounded,
\begin{align*}
    \int \lVert p_\star \nabla \phi \rVert dx_\star
    \leq \lVert \nabla \phi \rVert_\infty \int p_\star dx_\star < \infty .
\end{align*}
Therefore, $p_\star \nabla \phi \in L_1(\mathrm{Leb}^d)$.
Observe $\nabla \cdot (p_\star \nabla \phi) = p_\star \nabla^2 \phi + \langle \nabla \phi, \nabla p_\star \rangle$.
Since $\nabla^2 \phi$ is bounded, the first term is in $L_1$.
The second term is bounded in $L_1$ since
\begin{align*}
    \int \lVert \nabla p_\star \rVert dx_\star
    =
    \int \frac{\lVert \nabla p_\star \rVert}{p_\star} p_\star dx_\star
    \leq \left(\int \lVert \nabla \log p_\star \rVert^2 p_\star d x_\star\right)^{1/2}
    < \infty .
\end{align*}
The case for $j = 2$ follows by observing that
\begin{align*}
    - \int \nabla \phi \otimes \nabla p_\star dx_\star = \int \phi \nabla^2 p_\star dx_\star
    \Leftrightarrow
    \int \nabla  (\phi \nabla p_\star) = 0 ,
\end{align*}
and repeating the same steps above, observing that $\nabla \cdot (\phi \nabla p_\star)$ involves $\nabla \log p_\star$ and $\nabla^2 \log p_\star$.
The same argument applies to any $m$ and $j \in \{1, \cdots, m\}$.

Let $H_m (x) = (-1)^m e^{x^2/2} \nabla^m e^{-x^2/2}$ be the $m$-th order Hermite polynomial.
It holds by the marginalization $p_t (x) = \int p_{t|x} (x | x_\star) p_\star (x_\star) dx_\star$ and the identity
\begin{equation}\label{eq:gradm_pt_conditional}
    \nabla^m p_{t | 0} (x | x_\star)
    = \frac{(-1)^m}{\sigma_t^m} H_m \left(\frac{x - \alpha_t x_\star}{\sigma_t} \right) p_{t|0} (x | x_\star)
\end{equation}
that
\begin{align*}
    \frac{\nabla^m p_t (x)}{p_t (x)}
    &=
    \int \frac{(-1)^m}{\sigma_t^m}
    H_m (z) \frac{p_{t|0} (x | x_\star)}{p_t (x)} p_\star (x_\star) dx_\star
    \\ & =
    \frac{(-1)^m}{\sigma_t^m} \mathbb{E} \left[H_m (Z) | X_t = x \right] ,
    \numberthis \label{eq:hermite_identity}
\end{align*}
where $m = 1$ recovers Tweedie's formula \citep{robbins56,esposito}
\begin{equation}\label{eq:tweedie}
    s_t (x)
    = \frac{\alpha_t \mathbb{E}[X_\star | X_t = x] - x}{\sigma_t^2}
    .
\end{equation}
Observe the identity
\begin{align*}
    \nabla_x p_{t | \star} (x | x_\star)
    = - \frac{1}{\alpha_t} \nabla_{x_\star} p_{t | \star} (x | x_\star) ,
\end{align*}
which holds by $p_{t | \star} (x | x_\star)$ being a Gaussian transition probability $\mathcal{N}(\alpha_t x_\star; \sigma_t I)$.
Applying the same derivatives $m$ times, we obtain
\begin{align*}
    \nabla_x^m p_{t | \star} (x | x_\star)
    = \frac{(-1)^m}{\alpha_t^m} \nabla_{x_\star}^m p_{t | \star} (x | x_\star) .
\end{align*}
Marginalizing, we obtain
\begin{align*}
    \nabla_x^m p_t (x)
    = \nabla_x^m \int p_{t | \star} (x | x_\star) p_\star (x_\star) dx_\star
    &=
    \int
    \frac{(-1)^m}{\alpha_t^m} \nabla_{x_\star}^m p_{t | \star} (x | x_\star) p_\star (x_\star) dx_\star
\end{align*}
Using the integration by parts property \eqref{eq:vanishing_boundary} repeatedly, we obtain
\begin{align*}
    \frac{(-1)^m}{\alpha_t^m}     \int
\nabla_{x_\star}^m p_{t | \star} (x | x_\star) p_\star (x_\star)  dx_\star
    &= \frac{(-1)^{m-1}} {\alpha_t^m}\int \nabla^{m-1}p_{t | \star} (x | x_\star) \otimes \nabla p_\star (x_\star)  dx_\star
    \\ &= \frac{1}{\alpha_t^m} \int p_{t | \star} (x | x_\star) \nabla^m p_\star (x_\star) dx_\star .
\end{align*}

Dividing by $p_t (x)$, we obtain the identity
\begin{align*}
    \frac{\nabla^m p_t (x)}{p_t (x)}
    &= \frac{1}{\alpha_t^m} \int \frac{\nabla^m p_\star (x_\star)}{p_\star (x_\star)} \frac{p_{t|\star} (x | x_\star) p_\star (x_\star)}{p_t (x)} dx_\star
    \\ &
    = \alpha_t^{-m} \mathbb{E}\left[ \frac{\nabla^m p_\star (X_\star)}{p_\star (X_\star)} \mid X_t = x\right] .
\end{align*}
To summarize, we obtained above and in \eqref{eq:hermite_identity} the two identities
\begin{align*}
    \frac{\nabla^m p_t (x)}{p_t (x)}
    & = \alpha_t^{-m}\mathbb{E} \left[ \frac{\nabla^m p_\star (X_\star)}{p_\star (X_\star)} \mid X_t = x \right] ,
    \\
    \frac{\nabla^m p_t (x)}{p_t (x)}
    & = \sigma_t^{-m} (-1)^m \mathbb{E} \left[ H_m (Z) \mid X_t = x \right] .
\end{align*}
For any $\theta$, it holds that their linear combination is equal to $\nabla^m p_t (x) / p_t (x)$.
By Jensen's inequality and $X_\star \perp Z$, we obtain
\begin{align*}
    \left\| \frac{\nabla^m p_t}{p_t} \right\|_{L_2}^2
    &\leq
    \mathbb{E}\left\| (1 - \theta) \alpha_t^{-m}\frac{\nabla^m p_\star (X_\star)}{p_\star (X_\star)}
    + \theta \sigma_t^{-m} (-1)^m H_m (Z) \right\|^2
    \\ &
    =
    (1 - \theta)^2 \alpha_t^{-2m} \left\| \frac{\nabla^m p_\star}{p_\star} \right\|_{L_2}^2
    + \theta^2 \sigma_t^{-2m} \left\| H_m \right\|_{L_2(\gamma^d)}^2 ,
\end{align*}
where the cross term vanishes because $Z \perp X_\star$ and $\mathbb{E} H_m (Z) = 0$.
Minimizing over $\theta$, we obtain \eqref{eq:generalized_stam}.
\end{proof}

A consequence of Stam's inequality is that $\alpha_T^2 \lVert X_\star \rVert^2_{L_2} \ll \sigma_T^2 d$ implies $\fisher_T \approx \sigma_T^{-2} d$.
Recall Tweedie's formula \eqref{eq:tweedie}, which implies
\begin{align*}
    \fisher_T = \frac{1}{\sigma_T^2} \mathbb{E} \lVert \mathbb{E}[Z | X_T] \rVert^2
    &= \frac{1}{\sigma_T^2}\left( d - \mathbb{E} \mathrm{Tr} \mathrm{Cov}(Z | X_T) \right) .
\end{align*}
Using $Z = \sigma_T^{-1}(X_T - \alpha_T X_\star)$, we have
\begin{align*}
    \mathbb{E}\tr \mathrm{Cov}(Z | X_T) = \frac{\alpha_T^2}{\sigma_T^2} \mathbb{E}\tr \mathrm{Cov}(X_\star | X_T)
    \leq
    \frac{\alpha_T^2}{\sigma_T^2} \tr \mathrm{Cov}(X_\star)
    .
\end{align*}
Substituting for Tweedie's formula above, we have
\begin{equation}\label{eq:JT_squeeze}
    \frac{1}{\sigma_T^2} \left(d - \frac{\alpha_T^2}{\sigma_T^2} \lVert X_\star \rVert_{L_2}^2 \right)
    \leq
    \fisher_T
    \leq
    \frac{d}{\sigma_T^2}
    ,
\end{equation}
which implies $\fisher_T \approx \sigma_T^{-2} d$ when $\alpha_T^2 \lVert X_\star \rVert^2_{L_2} \ll \sigma_T^2 d$.

\subsection{Score Function Dynamics Along the Reverse SDE}\label{proof:dst}
The SDE describing the score function $s_{T-\tau} (X_\tau^\leftarrow) = \nabla \log p_{T-\tau} (X_\tau^\leftarrow)$ along the reverse process $(X_\tau^\leftarrow)$ with constant schedules was established in \citep[Proposition 3.2]{conforti2024klconvergenceguaranteesscore}.
Assuming the regularity conditions in \Cref{ass:fisher_information}, we derive the corresponding SDE for time-varying schedules.
\begin{lemma}\label{prop:dst}
    Under \Cref{ass:fisher_information}, it holds that
    \begin{equation*}
    ds_{T-\tau} (X_\tau^\leftarrow)
    =
    - f(T-\tau) s_{T-\tau}(X_\tau^\leftarrow) d\tau
    + \sqrt{g(T-\tau)} \nabla s_{T-\tau} (X_\tau^\leftarrow) dB_\tau
    \end{equation*}
    is well-defined, and $\lVert s_t \rVert_{L_2}, \lVert \nabla s_t \rVert_{L_2} < \infty$ for all $t \in [0, T]$.
\end{lemma}

\begin{proof}
We use $t = T - \tau$.
The Fokker-Planck equation governing the evolution of the density $p_t$ along the SDE \eqref{eq:dXt} is
\begin{equation}\label{eq:fokker_planck}
    \frac{\partial}{\partial t} p_t = \nabla \cdot (f(t) x p_t) + \frac{1}{2} g(t) \nabla \cdot \nabla p_t .
\end{equation}
By the chain rule, we obtain the point-wise evaluation
\begin{align*}
    \frac{\partial}{\partial t} \log p_t = f(t) d + f(t)\langle x, s_t \rangle + \frac{1}{2} g(t) (\nabla \cdot s_t + \lVert s_t \rVert^2) .
\end{align*}
Using the identities $\nabla (x \cdot s_t) = s_t + (\nabla s_t) x$, $\nabla \lVert s_t \rVert^2 = 2 (\nabla s_t) s_t$ and taking the gradient on both sides of the above equation, we obtain
\begin{equation}\label{eq:partialt_st}
    \partial_t s_t
    =
    f(t) s_t + f(t) (\nabla s_t) x + \frac{1}{2} g(t) \nabla(\nabla \cdot s_t) + g(t) (\nabla s_t) s_t ,
\end{equation}
where $(\nabla s_t) x$ is understood as the matrix-vector product of the Jacobian $\nabla s_t$ and the vector $x$.
Using the \ito formula, we obtain that the drift driving $ds_{T-\tau} (X_\tau^\leftarrow)$ along the reverse process is
\begin{align*}
    & \quad  \left(-\partial_t s_t\right) (X_\tau^\leftarrow)
    +
    \nabla s_t(X_\tau^\leftarrow)\left[ f(t) X_\tau^\leftarrow + g(t) s_t (X_\tau^\leftarrow)
    \right]
    +
    \frac{g(t)}{2} \nabla (\nabla \cdot s_t) (X_\tau^\leftarrow)
    \\ &=
    - f(t) s_t (X_\tau^\leftarrow) .
\end{align*}
Substituting for the \ito formula, we obtain
\begin{align*}
    ds_{T-\tau} (X_\tau^\leftarrow)
    &=
    - f(t) s_t(X_\tau^\leftarrow) d\tau
    + \sqrt{g(t)} (\nabla s_t) (X_\tau^\leftarrow) dB_\tau .
\end{align*}
Since $\nabla s_t =  \frac{\nabla^2 p_t}{p_t} - \frac{\nabla p_t \nabla p_t^\dagger}{p_t^2}$, the finiteness of $\lVert \nabla s_t \rVert_{L_2}$ follows from \Cref{lem:generalized_stam}.
\end{proof}

\subsection{Regularity Properties of the Score Function}\label{app:higher_order}

Gaussian mixture models form an important class of distributions that is often used to model real-world data.
We show that \Cref{ass:higher_order} is satisfied by a Gaussian mixture model.
Consider $p_0 (x) = \sum_{i=1}^\ell w_i \mathcal{N}(x; \mu_i, \Sigma_i)$, where each $\Sigma_i$ is positive definite.
The score function $s_0 (x)$ can be evaluated as in \citep{chen2024learning}, where
\begin{align*}
    s_0 (x) = - \sum_{i=1}^\ell \pi_i (x) \Sigma_i^{-1} (x - \mu_i) ,
    \quad
    \pi_i (x) = \frac{w_i \mathcal{N}(x; \mu_i, \Sigma_i)}{p_0 (x)} .
\end{align*}
Taking the gradient on both sides, we obtain
\begin{align*}
    \nabla s_0 (x)
    = - \sum_{i=1}^\ell \pi_i (x) \Sigma_i^{-1}
    + \sum_{i=1}^\ell
    \pi_i (x) (- \Sigma_i^{-1} (x - \mu_i) - s_0 (x)) (-\Sigma_i^{-1} (x - \mu_i) - s_0 (x))^\dagger .
\end{align*}
Since the score function of each component is affine and $0 \leq \pi_i(x) \leq 1$, the score and its first two derivatives grow at most polynomially in $\lVert x \rVert$.
A finite Gaussian mixture has finite moments of every order, so the required $L_2$ norms are finite.
For $t > 0$, $p_t$ remains a finite Gaussian mixture, and the same argument applies.

\subsection{Fisher Information Dynamics}\label{proof:fisher_ode}
The following states the Fisher information dynamics in \eqref{eq:fisher_ode}.
\begin{lemma}\label{lem:fisher_ode}
    Suppose \Cref{ass:fisher_information} holds.
    For any $(f, g) \in \Theta$,
    \begin{equation*}
        \dot{\fisher}_t = 2f(t) \fisher_t - g(t)
        \lVert \nabla s_t \rVert^2_{L_2}
        .
    \end{equation*}
\end{lemma}
\begin{proof}
In this proof, we use the notation $p_t^X = p_t$ to distinguish between the density $p_t^X$ of the process $(X_t)$ from the process $(Y_t)$ defined as
\begin{equation}\label{eq:XvsY}
    X_t = \alpha_t Y_t.
\end{equation}
By the \ito formula, the two processes are related through the SDE $dX_t = \dot{\alpha}_t Y_t dt + \alpha_t dY_t$.
Substituting $\dot{\alpha}_t = - f(t) \alpha_t$ and using \eqref{eq:XvsY}, we obtain
\begin{align*}
    dY_t = - \frac{\dot{\alpha}_t}{\alpha_t} Y_t dt + dX_t
    &=
    f(t) Y_t dt - f(t) \frac{X_t}{\alpha_t} dt + \frac{\sqrt{g(t)}}{\alpha_t} dB_t
    = \frac{\sqrt{g(t)}}{\alpha_t} dB_t .
    \numberthis \label{eq:dYt}
\end{align*}
The density $p_t^Y$ of $Y_t$ is a convolution between $p_0^Y$ and a Gaussian kernel, and differentiating under the integral yields the well-known identity (e.g., \citep[Section 5]{stamInequalitiesSatisfiedQuantities1959})
\begin{align*}
    \partial_t p_t^Y = \frac{g(t)}{2 \alpha_t^2} \Delta p_t^Y .
\end{align*}
The Bochner formula \citep{villaniShortProofConcavity2000} yields the identity
\begin{equation}\label{eq:debrujin}
    \frac{d}{dt} \fisher(p_t^Y)
    = - \frac{g(t)}{\alpha_t^2} \lVert \nabla^2 \log p_t^Y \rVert_{L_2}^2 .
\end{equation}
Moreover, using the change of variables $p_t^X (x) = \alpha_t^{-d} p_t^Y (x/\alpha_t) \Rightarrow \nabla \log p_t^X (x) = \alpha_t^{-1} \nabla \log p_t^Y (x/\alpha_t)$, we obtain the relation between the Fisher information of the two laws:
\begin{equation}\label{eq:IXIy}
    \fisher(p_t^X) = \alpha_t^{-2} \fisher(p^Y_t)    .
\end{equation}
Differentiating $\fisher(p_t^X)$ in \eqref{eq:IXIy}, we obtain
\begin{align*}
    \frac{d}{dt} \fisher(p^X_t)
    &= \left[-2 \alpha^{-3} \dot{\alpha}_t \right] \fisher(p_t^Y) + \alpha_t^{-2} \frac{d}{dt} \fisher(p_t^Y)
    \\ &
    =
    2 f(t) \alpha_t^{-2} \fisher(p_t^Y) + \alpha_t^{-2}\frac{d}{dt} \fisher(p_t^Y)
    .
\end{align*}
Substituting \eqref{eq:IXIy} and $\nabla_x^2 \log p_t^X (x) = \alpha_t^{-2} \nabla^2 \log p_t^Y (x/\alpha_t)$, we obtain
\begin{align*}
    \frac{d}{dt} \fisher(p^X_t)
    =
    2f(t) \fisher(p_t^X)
    - \frac{g(t)}{\alpha_t^{4}(t)}\lVert \nabla^2 \log p_t^Y \rVert_{L_2}^2
    =
    2f(t) \fisher(p_t^X)
    - g(t)\lVert \nabla^2 \log p_t^X \rVert_{L_2}^2  .
\end{align*}
\end{proof}

We use the following ``concavity of entropy power'' theorem \citep{demboSimpleProofConcavity1989,villaniShortProofConcavity2000}.
\begin{lemma}\label{lem:st_Fisher_lb}
    For any $X_\star$ with density $p_\star$ satisfying \Cref{ass:fisher_information}, it holds that
    \begin{align*}
        \fisher_t^2 \leq d \lVert \nabla s_t \rVert^2_{L_2}.
    \end{align*}
\end{lemma}
\begin{proof}

Recall that the solution $(X_t)$ to \eqref{eq:dXt} can be written as \eqref{eq:Xt}.
Because the density $p_t$ is obtained as a convolution of $p_\star$ with a Gaussian density, $p_t$ is infinitely differentiable and the integration by parts formula
\begin{align*}
    \mathbb{E} \nabla \log p_t \nabla \log p_t^\dagger = -\mathbb{E} \nabla^2 \log p_t , \quad \forall t \in (0, T]
\end{align*}
holds.
The relation $\lVert \nabla s_t \rVert^2_{L_2, F} \geq \fisher_t^2/d$ is shown next:
\begin{align*}
    \fisher_t = \mathbb{E} \lVert \nabla \log p_t \rVert^2
    = \mathrm{Tr} \mathbb{E} \nabla \log p_t \nabla \log p_t^\dagger
    &\underset{(*)}{=} - \mathrm{Tr} \mathbb{E} \nabla^2 \log p_t \\
    &= - \mathrm{Tr} \mathbb{E} \nabla s_t ,
\end{align*}
where $(*)$ follows from integration by parts.
By Jensen's inequality and the Cauchy-Schwartz inequality,
\begin{align*}
    \fisher_t^2 = (\mathbb{E} \mathrm{Tr} \nabla s_t)^2
    \leq \mathbb{E} (\mathrm{Tr} \nabla s_t )^2
    &\leq d \mathrm{Tr} \mathbb{E} (\nabla s_t)^2
    \coloneqq  d \lVert \nabla s_t \rVert^2_{L_2}
    .
\end{align*}
\end{proof}

\section{Proofs of \texorpdfstring{\Cref{thm:ub,thm:sampling_error}}{the Main Bounds} (\texorpdfstring{\Cref{sec:opt-control}}{Optimal Control})}\label{app:main_results}

\subsection{Proof of \texorpdfstring{\Cref{thm:ub}}{the Sampling Error Bound}}\label{proof:ub}
We first prove the bound on the initialization error.
When $(f, g)$ are constants, the log-Sobolev inequality \citep{grossLogarithmicSobolevInequalities1975} and the contraction of the KL divergence along the Ornstein-Uhlenbeck process \citep{bakry2013analysis} can be used to obtain a bound as in \citep{conforti2024klconvergenceguaranteesscore}.
A different bound can be obtained as in \citep[Proposition 4]{benton2024nearly} and \citep[Lemma 9]{hchen_improved_score} based on the convexity of the KL divergence.
For time-varying schedules, we find it to be more useful to utilize a small modification to the bound emerging from the latter argument.

By \eqref{eq:Xt}, we have $(X_T | X_\star) \sim \mathcal{N}(\alpha_T X_\star, \sigma_T^2 I)$.
Let $\mi$ denote the mutual information so that $\mi(X_\star ; X_T) = \mathbb{E}_{p_\star} [\mathrm{KL}(\mathcal{N}(\alpha_T X_\star, \sigma_T^2 I) \| p_T)]$.
Using the chain rule, we obtain the bound
\begin{align*}
    \mathrm{KL}(\mathrm{Law}(X_T) \| \mathcal{N}(0, \sigma_T^2 I))
    &= \mathbb{E}_{p_\star} \left[ \mathrm{KL}( \mathrm{Law}(X_T | X_0) \| \mathcal{N}(0, \sigma_T^2 I)) \right] - \mi(X_0 ; X_T)
    \\ &\leq
    \mathbb{E}_{p_\star} \left[\mathrm{KL}(\mathcal{N}(\alpha_T X_\star, \sigma_T^2 I) \| \mathcal{N}(0, \sigma_T^2 I))\right]
    \\ &
    = \frac{\alpha_T^2}{2 \sigma_T^2} \lVert X_\star \rVert^2_{L_2 (p_\star)}
    .
\end{align*}
where the inequality follows from non-negativity of $\mi$.

Recall the Girsanov bound \eqref{eq:girsanov}.
We use \Cref{prop:dst} to obtain the identity
\begin{equation}
    s_t (X_\tau^\leftarrow) - s_{T-\tau_k} (X_{\tau_k}^\leftarrow)
    =
    -\int_{\tau_k}^\tau \left[f(\tilde{t}) s_{\tilde{t}} (X_{\tilde{\tau}}^\leftarrow) \right] d\tilde{\tau}
    +
    \int_{\tau_k}^\tau \sqrt{g(\tilde{t})} \nabla s_{\tilde{t}} (X_{\tilde{\tau}}^\leftarrow) dB_{\tilde{\tau}} ,
\end{equation}
where $t=T-\tau$ and $\tilde{t} = T - \tilde{\tau}$.
Evaluating the squared $L_2$ norm, we have the bound
\begin{align*}
    &\quad  \left\lVert s_t (X_\tau^\leftarrow) - s_{T-\tau_k} (X_{\tau_k}^\leftarrow) \right\rVert_{L_2}^2
    \\ &\leq
    2 (\tau - \tau_k)
    \int_{\tau_k}^\tau f(\tilde{t})^2 \left\lVert s_{\tilde{t}} (X_{\tilde{\tau}}^\leftarrow) \right\rVert_{L_2}^2 d\tilde{\tau}
    +
    2 \int_{\tau_k}^\tau
    g(\tilde{t}) \left\lVert \nabla s_{\tilde{t}} (X_{\tilde{\tau}}^\leftarrow) \right\rVert_{L_2}^2 d\tilde{\tau}
    \\ &
    \leq 2 h \int_{\tau_k}^{\tau} f(\tilde{t})^2 \fisher_{\tilde{t}} d\tilde{\tau}
    +
    2 \int_{\tau_k}^\tau
    g(\tilde{t}) \left\lVert \nabla s_{\tilde{t}} (X_{\tilde{\tau}}^\leftarrow) \right\rVert_{L_2}^2 d\tilde{\tau}
    \numberthis \label{eq:score_increment_before_absorption}
    \\ &
    \lesssim
    \int_{\tau_k}^\tau
    g(\tilde{t}) \lVert \nabla s_{\tilde{t}} (X_{\tilde{\tau}}^\leftarrow) \rVert_{L_2}^2
    d\tilde{\tau} ,
\end{align*}
where the last holds when $h$ is sufficiently small such that $h < 1$ and satisfies
\begin{equation}\label{eq:stepsize}
h
\leq
\min_{k \leq n - 1} \min_{\tau \in [\tau_k, \tau_{k+1}]}
\frac{\int_{\tau_k}^\tau g(\tilde{t}) \lVert \nabla s_{\tilde{t}} (X_{\tilde{\tau}}^\leftarrow) \rVert_{L_2}^2
d\tilde{\tau}}{    \int_{\tau_k}^\tau f(\tilde{t})^2 \fisher_{\tilde{t}} d\tilde{\tau}}
.
\end{equation}
We show that $h \lesssim 1/n$ satisfies \eqref{eq:stepsize} in \Cref{prop:stepsize}.

Since $g$ is positive and continuous on $[0,T]$, for sufficiently small $h$, $g(T-\tau)\leq 2g(T-\tilde\tau)$ whenever $\lvert\tau-\tilde\tau\rvert\leq h$.
Integrating the preceding inequality and interchanging the order of integration, we obtain
\begin{align*}
    &\int_{\tau_k}^{\tau_{k+1}} g(t)
    \left\lVert s_t (X_\tau^\leftarrow) - s_{T-\tau_k} (X_{\tau_k}^\leftarrow) \right\rVert_{L_2}^2 d\tau
    \\ &\lesssim
    \int_{\tau_k}^{\tau_{k+1}} g(T-\tau)
    \int_{\tau_k}^\tau g(T-\tilde\tau)
    \lVert\nabla s_{T-\tilde\tau}(X_{\tilde\tau}^\leftarrow)\rVert_{L_2}^2\,d\tilde\tau\,d\tau
    \\ &\leq
    2\int_{\tau_k}^{\tau_{k+1}}(\tau_{k+1}-\tilde\tau)g(T-\tilde\tau)^2
    \lVert\nabla s_{T-\tilde\tau}(X_{\tilde\tau}^\leftarrow)\rVert_{L_2}^2\,d\tilde\tau
    \\ &\leq
    2h\int_{\tau_k}^{\tau_{k+1}}g(t)^2
    \lVert\nabla s_t(X_\tau^\leftarrow)\rVert_{L_2}^2\,d\tau
    \numberthis \label{eq:score_matching_bound}
    .
\end{align*}

By $X_{\tilde{t}} \overset{d.}{=} X_{\tilde{\tau}}^\leftarrow$, it holds that $\lVert \nabla s_t (X_{\tau}^\leftarrow) \rVert_{L_2} = \lVert \nabla s_t (X_t) \rVert_{L_2}$ and we denote this quantity by $\lVert \nabla s_t \rVert_{L_2}$.
Recall \eqref{eq:fisher_ode}, where
\begin{align*}
    \dot{\fisher}_t = 2 f(t) \fisher_t - g(t) \lVert \nabla s_t \rVert^2_{L_2}
    .
\end{align*}
Substituting $g(t) \lVert \nabla s_t \rVert^2_{L_2} = 2f(t) \fisher_t - \dot{\fisher}_t$ into the integrated bound \eqref{eq:score_matching_bound}, we obtain
\begin{align*}
    \int_{\tau_k}^{\tau_{k+1}} g(t) \left\lVert s_t (X_\tau^\leftarrow) - s_{T-\tau_k} (X_{\tau_k}^\leftarrow) \right\rVert_{L_2}^2 d\tau
    &\lesssim
    h \int_{\tau_k}^{\tau_{k+1}} g(t) \left(2f(t) \fisher_t - \dot{\fisher}_t    \right) d \tau
    \\ &
    =
    h \int_{\tau_k}^{\tau_{k+1}} \frac{\left(2f(t) \fisher_t - \dot{\fisher}_t    \right)^2}{\lVert \nabla s_t \rVert_{L_2}^2} d \tau
    \\ &
    \leq
    h d \int_{\tau_k}^{\tau_{k+1}}
    \left(2f(t) - \frac{\dot{\fisher}_t}{\fisher_t} \right)^2 d\tau
    .
\end{align*}
Summing over all $k$, we obtain
\begin{align*}
    \sum_{k=0}^{n-1}\int_{\tau_k}^{\tau_{k+1}} g(t) \left\lVert s_t (X_\tau^\leftarrow) - s_{T-\tau_k} (X_{\tau_k}^\leftarrow) \right\rVert_{L_2}^2 d\tau
    &\lesssim
    h d \int_{0}^T
    \left(2f(t) - \frac{\dot{\fisher}_t}{\fisher_t} \right)^2 d t
\end{align*}

\subsection{Discretization Step Size}\label{proof:stepsize}
\begin{proposition}\label{prop:stepsize}
    Under \Cref{ass:fisher_information}, it holds for any $(f, g) \in \Theta$ that any $h$ satisfying
    \begin{equation}\label{eq:stepsize_sufficient}
    \begin{aligned}
        h \leq
        \frac{g_{\min}}{f_{\max}^2} \min_{t \in [0, T]}  \frac{1}{\alpha_t^2 \lVert X_\star \rVert_{L_2}^2 /d + \sigma_t^2}
    \end{aligned}
    \end{equation}
    necessarily satisfies for any $\{\tau_k\}_{k=0}^{n}$ with $h = \tau_{k+1} - \tau_k$ and $\tau_n = T$ the inequality \eqref{eq:stepsize}.
\end{proposition}
\begin{proof}
Recall the inequality \eqref{eq:stepsize}:
\begin{align*}
    h
    \leq
    \min_{k} \min_{\tau \in [\tau_k, \tau_{k+1}]}
    \frac{\int_{\tau_k}^\tau g(\tilde{t}) \lVert \nabla s_{\tilde{t}} (X_{\tilde{\tau}}^\leftarrow) \rVert_{L_2}^2
    d\tilde{\tau}}{    \int_{\tau_k}^\tau f(\tilde{t})^2 \fisher_{\tilde{t}} d\tilde{\tau}}
\end{align*}
We obtain a lower bound on
\begin{align*}
    \frac{\int_{\tau_k}^\tau
    g(\tilde{t}) \lVert \nabla s_{\tilde{t}} (X_{\tilde{\tau}}^\leftarrow) \rVert_{L_2}^2
    d\tilde{\tau}}{    \int_{\tau_k}^\tau f(\tilde{t})^2 \fisher_{\tilde{t}} d\tilde{\tau}} ,
\end{align*}
such that any $h$ chosen to be smaller than the lower bound necessarily satisfies \eqref{eq:stepsize}.

By \Cref{lem:st_Fisher_lb}, it holds that $\lVert \nabla s_t \rVert_{L_2} \geq \frac{\fisher_t}{\sqrt{d}} $, and we have
\begin{align*}
    \frac{\int_{\tau_k}^\tau
    g(\tilde{t}) \lVert \nabla s_{\tilde{t}} \rVert_{L_2}^2
    d\tilde{\tau}}{    \int_{\tau_k}^\tau f(\tilde{t})^2 \fisher_{\tilde{t}} d\tilde{\tau}}
    &\geq
    \frac{1}{d}\frac{\int_{\tau_k}^\tau
    g(\tilde{t}) \lVert \nabla s_{\tilde{t}} \rVert_{L_2} \fisher_{\tilde{t}}
    d\tilde{\tau}}{    \int_{\tau_k}^\tau f(\tilde{t})^2 \lVert \nabla s_{\tilde{t}} \rVert_{L_2} d\tilde{\tau}}
    \\ &
    \geq
    \left(\frac{\min_{\tilde{t} \in [\tau_k, \tau]} \fisher_{\tilde{t}}}{d} \right)
    \frac{\int_{\tau_k}^\tau
    g(\tilde{t}) \lVert \nabla s_{\tilde{t}} \rVert_{L_2}
    d\tilde{\tau}}{    \int_{\tau_k}^\tau f(\tilde{t})^2 \lVert \nabla s_{\tilde{t}} \rVert_{L_2} d\tilde{\tau}}
    \numberthis \label{eq:grads_fisher_ratio}
    .
\end{align*}
The Cram\'{e}r-Rao lower bound (CRLB) states that a random variable $X$ whose density $p$ is defined against the Lebesgue measure satisfies
\begin{align*}
    \mathrm{Cov}(X) \succeq
    \left(\mathbb{E} \nabla \log p \nabla \log p^\dagger\right)^{-1} .
\end{align*}
Let $A \in \mathbb{R}^{d \times d}$ be a positive definite matrix with eigenvalues $\{\lambda_i\}$.
By the arithmetic-harmonic mean (AM-HM) inequality,
\begin{align*}
    \frac{1}{d} \sum_{i=1}^d \lambda_i \geq \frac{d}{\sum_{i=1}^d \frac{1}{\lambda_i}} .
\end{align*}
Rearranging, we obtain that $(\tr A)(\tr A^{-1}) \geq d^2$.
Using the above property, we have the inequality
\begin{align*}
    \lVert X \rVert_{L_2}^2 \fisher(p)
    \geq \tr \mathrm{Cov}(X) \tr \mathbb{E} \nabla \log p \nabla \log p^\dagger
    \geq d^2 .
\end{align*}
Rearranging, we obtain the following form of the CRLB:
\begin{equation}\label{eq:crlb}
    \fisher(p) \geq \frac{d^2}{\lVert X \rVert_{L_2}^2}.
\end{equation}
Substituting for the denominator in \eqref{eq:grads_fisher_ratio}, we obtain
\begin{align*}
    \frac{\int_{\tau_k}^\tau
    g(\tilde{t}) \lVert \nabla s_{\tilde{t}} \rVert_{L_2}^2
    d\tilde{\tau}}{    \int_{\tau_k}^\tau f(\tilde{t})^2 \fisher_{\tilde{t}} d\tilde{\tau}}
    &\geq
    \left(\min_{\tilde{t} \in [\tau_k, \tau]}  \frac{d}{\lVert X_t \rVert_{L_2}^2} \right)
    \frac{\int_{\tau_k}^\tau
    g(\tilde{t}) \lVert \nabla s_{\tilde{t}} \rVert_{L_2}
    d\tilde{\tau}}{    \int_{\tau_k}^\tau f(\tilde{t})^2 \lVert \nabla s_{\tilde{t}} \rVert_{L_2} d\tilde{\tau}}
    \\ &
    \geq
    \frac{g_{\min}}{f_{\max}^2}
    \min_{\tilde{t} \in [\tau_k, \tau]}  \frac{d}{\lVert X_t \rVert_{L_2}^2}
    .
\end{align*}
Minimizing both sides over $k$ and $\tau$, we obtain that a sufficient condition for $h$ to satisfy \eqref{eq:stepsize} is
\begin{align*}
    h \leq
    \frac{g_{\min}}{f_{\max}^2} \min_{t \in [0, T]}  \frac{d}{\lVert X_t \rVert_{L_2}^2}
    .
\end{align*}

\end{proof}

\subsection{Solving the Euler-Lagrange Equation}\label{proof:euler_lagrange}
The discretization error bound in \Cref{thm:ub} is minimized over $g$.
Under a change of variables $u(t) = \log \fisher_t$, the second term's integral in \eqref{eq:sampling_error} is given by
\begin{align*}
    \int_0^T \left[2f(t) - \dot{u}(t) \right]^2 dt .
\end{align*}
To find the optimal $g^*$, we first solve for the optimal trajectory:
\begin{equation}
    \min_{u(t)} \int_0^T \left(2 f(t) - \dot{u}(t) \right)^2 dt ,
    \quad u(0) = \log \fisher_\star ,
    \quad u(T) = \log \fisher_T .
\end{equation}
The integrand $\mathcal{L}(t, u, \dot{u}) = (2 f(t) - \dot{u})^2$ is jointly convex in $(u, \dot{u})$, and the optimality condition is given by the Euler--Lagrange equation
\begin{align*}
    \frac{\partial \mathcal{L}}{\partial u} - \frac{d}{dt} \left(\frac{\partial \mathcal{L}}{\partial \dot{u}} \right) = 0 .
\end{align*}
Evaluating $\partial_u \mathcal{L} = 0, \partial_{\dot{u}} \mathcal{L} = -2 (2f(t) - \dot{u}) = 2 \dot{u} - 4f(t)$, we obtain the optimality condition
\begin{align*}
    \frac{d}{dt} (\dot{u}_t - 2 f(t)) = 0 .
\end{align*}
In other words, the optimality condition is given by \eqref{eq:optimal_fisher}, which we re-state:
\begin{equation*}
    2 f(t) - \frac{\dot{\fisher}_t}{\fisher_t} = \lambda ,
\end{equation*}
where the value of $\lambda$ is determined by the boundary condition
\begin{equation}\label{eq:lambda_boundary_condition}
    \lambda T = 2 \int_0^T f(t) dt + \log \fisher_{\star} - \log \fisher_T .
\end{equation}
The above proof delineates an optimal trajectory $(\fisher_t^*)_{t \in [0, T]}$.
We substitute $\int_0^T f(t) dt = -\log \alpha_T$ for \eqref{eq:lambda_boundary_condition}, which yields
\begin{equation}\label{eq:lambda_boundary}
    \lambda = \frac{1}{T} \log \frac{\fisher_\star}{\alpha_T^2 \fisher_T} .
\end{equation}
Using $\fisher_T \leq d/\sigma_T^2$ from \eqref{eq:JT_squeeze} and the choice $\alpha_T^2/\sigma_T^2 = d / (\lVert X_\star \rVert^2_{L_2} n)$, we obtain
\begin{equation}\label{eq:lambda_lb}
    \lambda \geq \frac{1}{T} \log \frac{\fisher_\star \sigma_T^2}{d \alpha_T^2}
    = \frac{1}{T} \log \frac{\fisher_\star \| X_\star \|^2_{L_2} n}{d^2}
    ,
\end{equation}
which implies $\lambda \geq T^{-1} \log n$ by the CRLB \eqref{eq:crlb}.
Hence, $f_{\max} \leq (2T)^{-1} \log(n) \leq \lambda / 2$.

From \eqref{eq:IXIy}, the optimal trajectory $\fisher_t^*$ defines a corresponding trajectory $\fisher(p_t^Y) = \alpha_t^2 \fisher_t^*$.
Since $\alpha_t$ is a deterministic function of a fixed schedule $f(t)$, the optimal target $\fisher_t^*$ uniquely maps (as shown later) to a continuous variance trajectory $V_t^*$.
The noise schedule is then uniquely defined by differentiating this variance:
\begin{align*}
    g^* (t) = \alpha_t^2 \frac{d}{dt} V_t^* .
\end{align*}
From the Fisher ODE \eqref{eq:fisher_ode} and the optimality condition \eqref{eq:optimal_fisher}, we obtain the explicit expression for $g^*$:
\begin{equation}\label{eq:gstar_def}
    g^* (t)
    = \frac{2 f(t) \fisher_t^* - \dot{\fisher}_t^*}{\lVert \nabla s_t \rVert^2_{L_2}}
    = \frac{\lambda \fisher_t^*}{\lVert \nabla s_t \rVert^2_{L_2}}
    .
\end{equation}
Using $2f_{\max} \leq \lambda$, we obtain $g^*(t)\|\nabla s_t\|_{L_2}^2=\lambda\mathcal J_t^*$ is nonincreasing.

We next verify that $h=T/n$ is admissible.
By the assumption on $f_{\max}$ and the lower bound \eqref{eq:lambda_lb} on $\lambda$, we obtain
\[
    h f_{\max}^2\leq\frac{\log^2 n}{4Tn}\leq\frac{\log n}{T}\leq\lambda.
\]
Combining with \eqref{eq:gstar_def}, we obtain
\[
  h
  f(t)^2\fisher_t^*
  \leq h f_{\max}^2\fisher_t^*
  \leq g^*(t)\|\nabla s_t\|_{L_2}^2
  \Rightarrow
    h
    \leq \frac{g^*(t)\|\nabla s_t\|_{L_2}^2}{f(t)^2\fisher_t^*}
  .
\]
Therefore, $h = T/n$ satisfies \eqref{eq:stepsize}.
When $f(t) = 0$, the inequality for $h f(t)^2 \fisher_t^*$ implies that $h = T/n$ satisfies \eqref{eq:stepsize}.

Because \eqref{eq:score_matching_bound} was established for $h$ sufficiently small given $g$, we next show that the bound holds for $g^*$ in \eqref{eq:gstar_def} which may depend on $n$.
Using $f_{\max}^2 \leq \lambda /h$ and that $\lambda\fisher_t^*$ is nonincreasing in $t$, we obtain
\begin{align*}
    \left\lVert s_{T-\tau}(X_\tau^\leftarrow)
    -s_{T-\tau_k}(X_{\tau_k}^\leftarrow)\right\rVert_{L_2}^2
    &\leq 4\int_{\tau_k}^{\tau}
    g^*(T-\tilde\tau)
    \left\lVert\nabla s_{T-\tilde\tau}\right\rVert_{L_2}^2
    \,d\tilde\tau
    \\ &\leq 4h\,g^*(T-\tau)
    \left\lVert\nabla s_{T-\tau}\right\rVert_{L_2}^2.
\end{align*}
Multiplying by \(g^*(T-\tau)\) and using \( (\fisher_t^{*})^2 \leq d\|\nabla s_t\|_{L_2}^2\), we obtain
\begin{align*}
  \sum_k\int_{\tau_k}^{\tau_{k+1}}g^*(T-\tau)
  \left\|s_{T-\tau}(X_\tau^\leftarrow)
  -s_{T-\tau_k}(X_{\tau_k}^\leftarrow)\right\|_{L_2}^2\,d\tau
  &\leq 4h\lambda^2\int_0^T
  \frac{(\fisher_t^*)^2}{\|\nabla s_t\|_{L_2}^2}\,dt
  \\ &\leq 4hd\lambda^2T.
\end{align*}
Therefore, we have the bound
\begin{equation}\label{eq:sampling_error_lambda}
    \mathrm{KL}(p_\star \| \hat{p}_{\star})
    \leq \frac{\alpha_T^2}{2 \sigma_T^2} \lVert X_\star \rVert^2_{L_2 (p_\star)}
    + 2 hd \lambda^2 T .
\end{equation}

Substituting $h = T/n$, the lower bound \eqref{eq:JT_squeeze} on $\fisher_T$, and \eqref{eq:lambda_boundary} for $\lambda$, we obtain that the discretization error bound in \eqref{eq:sampling_error_lambda} is
\begin{align*}
    h d \lambda^2 T
    \lesssim
    \frac{d}{n} \left(\log \frac{\fisher_\star \sigma_T^2}{d\alpha_T^2}\right)^2
    =
    \frac{d}{n}\log^2 \left(\frac{\fisher_\star}{d} \cdot \frac{\| X_\star \|^2_{L_2} }{d} n \right)
    ,
\end{align*}
where the last step uses the choice of $\alpha_T^2/\sigma_T^2= d/(\| X_\star \|_{L_2}^2 n)$.
Adding the initialization error back, we obtain
\begin{equation}\label{eq:sampling_error_1}
    \mathrm{KL}(p_\star \| \hat{p}_\star) \lesssim
    \frac{d}{2n}
    +
    \frac{d}{n}\log^2 \left(\frac{\fisher_\star}{d} \cdot \frac{\| X_\star \|^2_{L_2} }{d} n \right)
    .
\end{equation}

\textbf{Existence:}
While the Euler-Lagrange optimization was solved with respect to $u_t = \log \fisher_t$, we must verify that the optimal trajectory $\fisher_t^*$ corresponds to a valid noise schedule $g^* \in \mathcal{G}$.
We establish this mapping using the process $Y_t = X_t / \alpha_t$ introduced in \Cref{proof:fisher_ode}.

We show that there exists a bijection $\fisher_t^* \mapsto g^*(t)$, which implies that the trajectory $(\fisher_t^*)$ can be controlled by a choice of $g^*$.
Using the relation $V_t \coloneqq \sigma_t^2 / \alpha_t^2 = \int_0^t g(\tilde{t})/\alpha_{\tilde{t}}^2 d \tilde{t}$ which implies $\dot{V}_t = g(t) / \alpha_t^2$, the identity \eqref{eq:debrujin} can be written as
\begin{equation}\label{eq:dJdV}
    \frac{d}{d V_t} \fisher (p_t^Y) = - \lVert \nabla^2 \log p_t^Y \rVert_{L_2}^2 .
\end{equation}
Because the RHS is strictly negative, $\fisher(p_t^Y)$ is strictly decreasing as a function of the variance $V_t$.
Consequently, the map $V_t \mapsto \fisher(p_t^Y)$ is a bijection.

Finally, we verify that $g^* (t) \in \mathcal G$ so that it satisfies the constraints in \eqref{eq:schedule_class}.
By \Cref{lem:generalized_stam}, $\lVert \nabla s_t \rVert_{L_2}^2 < \infty$ for all $t \in [0, T]$.
Since $\fisher_t > 0$ (e.g., via the CRLB \eqref{eq:crlb}), we conclude from \eqref{eq:gstar_def} that $g^* (t) > 0$.
By \eqref{eq:dJdV} and continuity of $\| \nabla^2 \log p_t^Y \|_{L_2}^2$, $g^*$ is continuous under \Cref{ass:first}.

This concludes the proof of \Cref{thm:sampling_error}.

\section{Proofs of \texorpdfstring{\Cref{thm:schedules,cor:optimal_schedules_practice}}{the Schedule Results} (\texorpdfstring{\Cref{sec:schedules}}{Noise Schedule Design})}\label{app:schedules}

\subsection{Higher Order Score and Fisher Information}\label{proof:st_Fisher_ub}
Under the strong log-concavity (SLC) condition in \Cref{ass:logconcave}, we have the reverse inequality of \Cref{lem:st_Fisher_lb},
\begin{lemma}\label{lem:st_Fisher_ub}
    Under \Cref{ass:logconcave}, it holds that
    \begin{align*}
        \lVert \nabla s_\star \rVert^2_{L_2} \leq \frac{1}{d} \left(\frac{M_\star}{m_\star} \right)^2\fisher_\star^2 .
    \end{align*}
\end{lemma}

Recall the integration by parts formula $\mathbb{E} \lVert s_\star \rVert^2 = - \mathbb{E} \mathrm{Tr} \nabla  s_\star$
By the log-concavity assumption, it holds that $-\mathrm{Tr} \mathbb{E} \nabla s_\star  \geq m_\star d$.
Moreover, the upper bound implies $\lVert \nabla s_\star \rVert^2_{L_2} = \mathbb{E} \mathrm{Tr} (\nabla s_\star)^2 \leq M_\star^2 d$.
Combining, we obtain
\begin{align*}
    \frac{\lVert \nabla s_\star \rVert^2_{L_2}}{\fisher_{\star}^2}
    =
    \frac{\mathbb{E} \mathrm{Tr} (\nabla s_\star)^2}{(\mathrm{Tr} \mathbb{E} \nabla s_\star)^2}
    \leq \frac{M_\star^2 d}{(m_\star d)^2}
    = \frac{1}{d} \left(\frac{M_\star}{m_\star} \right)^2
    .
\end{align*}

\subsection{Pr\'{e}kopa--Leindler inequality}\label{proof:pl_inequality}
The constants $(m_t, M_t)$ of the SLC condition in \Cref{ass:logconcave} are evaluated along the SDE \eqref{eq:dXt}.
For the following, denote by $p_Z$ the standard Gaussian density (of the random variable $Z$).
\begin{lemma}\label{lem:pl_inequality}
    Suppose the density $p_\star$ satisfies \eqref{eq:slc}.
    For any positive constants $(\alpha_t, \sigma_t)$, it holds that the density $p_t$ of $X_t = \alpha_t X_\star + \sigma_t Z$ satisfies
    \begin{equation}\label{eq:pl_inequality_quantitative}
        \mathrm{PS}\left(\frac{m(p_\star)}{\alpha_t^2}, \frac{m(p_Z)}{\sigma_t^2} \right) I  \preceq - \nabla^2 \log p_t \preceq
        \mathrm{PS}\left(\frac{M(p_\star)}{\alpha_t^2}, \frac{M(p_Z)}{\sigma_t^2} \right) I .
    \end{equation}
\end{lemma}
\begin{proof}

Let $p_{\star \mid t}(x_\star \mid x)$ be the conditional density of $X_\star$ given $X_t=x$, where
\begin{equation*}
    p_{\star \mid t}(x_\star \mid x) = 
    p_\star(x_\star)
    \exp\left(-\frac{\lVert x-\alpha_t x_\star\rVert^2}{2\sigma_t^2}\right).
\end{equation*}
By \Cref{ass:logconcave}, 
\begin{equation*}
    (m_\star+\alpha_t^2/\sigma_t^2)I
    \preceq 
    -\nabla^2 \log p_{\star \mid t}(x_\star \mid x) \preceq 
    (M_\star+\alpha_t^2/\sigma_t^2)I
    .
\end{equation*}
The Brascamp--Lieb \citep{brascamp1976extensions} and Cram\'{e}r--Rao inequalities therefore give
\begin{align*}
    \left(M_\star+\frac{\alpha_t^2}{\sigma_t^2}\right)^{-1}I
    \preceq \mathrm{Cov}(X_\star \mid X_t=x)
    \preceq \left(m_\star+\frac{\alpha_t^2}{\sigma_t^2}\right)^{-1}I .
\end{align*}
Using Tweedie's formula
\begin{align*}
    -\nabla^2\log p_t(x)
    = \frac{1}{\sigma_t^2}I
    - \frac{\alpha_t^2}{\sigma_t^4}\mathrm{Cov}(X_\star \mid X_t=x),
\end{align*}
we obtain
\begin{align*}
    \left(\frac{\alpha_t^2}{m_\star} + \sigma_t^2 \right)^{-1} I
    \preceq
    - \nabla^2 \log p_t (x)
    \preceq
    \left(\frac{\alpha_t^2}{M_\star} + \sigma_t^2 \right)^{-1} I
    .
\end{align*}
\end{proof}

\subsection{Equivalence Between Score Function's Jacobian and Fisher Information}\label{app:score_fisher_equivalence}
\begin{lemma}\label{lem:score_fisher_equivalence}
    Suppose \Cref{ass:fisher_information,ass:surrogate} hold.  Then
    \begin{equation}\label{eq:st_Fisher_equivalence_lemma}
    \frac{d}{\kappa^2} \leq  \frac{\fisher_t^2}{\lVert\nabla s_t\rVert^2_{L_2}} \leq d, \quad \forall t \in [0, T],
    \end{equation}
    where the constant $\kappa$ is defined as follows:
    \begin{itemize}[leftmargin=1.6em]
        \item When $p_\star$ is SLC under \Cref{ass:logconcave}: $\kappa = \frac{M_\star}{m_\star}$;
        \item When $p_\star$ is a GMM under \Cref{ass:gmm}:
        \begin{equation}
        \kappa = \left(1+ \frac{\Delta_\mu^2}{4 \nu^2} \right)^{1/2}
        \left(1 + \frac{1}{\nu^2} \cdot \frac{1}{d} \sum_{i=1}^L \pi_i \lVert \mu_i - \bar{\mu}\rVert^2 \right),
        \end{equation}
        where $\bar{\mu} = \sum_{i=1}^L \pi_i \mu_i$ and $\Delta_\mu = \max_{1 \leq i, \ell \leq L} \lVert \mu_i - \mu_\ell \rVert$.

        In particular, if $\max_{1 \leq i, \ell \leq L}\lVert \mu_i - \bar{\mu} \rVert \leq R_\mu$ for some $R_\mu > 0$, then
        \begin{align}
            \kappa \leq \left(1+ \frac{R_\mu^2}{\nu^2} \right) \left(1 + \frac{1}{d} \frac{R_\mu^2}{\nu^2} \right) .\label{eq:kappa-upper-gmm}
        \end{align}
    \end{itemize}
\end{lemma}

\begin{proof}
The upper bound $\frac{\fisher_t^2}{\lVert\nabla s_t\rVert^2_{L_2}} \leq d$ directly follows from \Cref{lem:st_Fisher_lb}, and it does not require \Cref{ass:surrogate}.

In the remainder of the proof, we focus on proving the lower bound $\frac{\fisher_t^2}{\lVert\nabla s_t\rVert^2_{L_2}} \ge \frac{d}{\kappa^2}$, or equivalently, $\frac{\lVert\nabla s_t\rVert^2_{L_2}}{\fisher_t^2} \le \frac{\kappa^2}{d}$.

\textbf{Log-Concave Distributions.}
Because $p_t$ remains log-concave for all $t$, its Hessian can be bounded in the positive-definite (L\"{o}wner) order $\prec$: $m_t I \preceq -\nabla^2 \log p_t \preceq M_t I$.
We define $\kappa_t = M_t / m_t$ as the condition number of this log-concave density.
Adding Gaussian noise decreases the condition number and hence, the time-varying condition number is bounded by that associated with the target density: $\kappa = \sup_{t} M_t / m_t = M_\star / m_\star$.

The above remark is described formally.
By \Cref{lem:st_Fisher_ub}, there exists some $(m_\star, M_\star)$ such that
\begin{equation*}
    \lVert \nabla s_\star \rVert^2_{L_2} \leq \frac{1}{d} \left(\frac{M_\star}{m_\star}\right)^2 \fisher_\star^2 .
\end{equation*}
By \Cref{lem:pl_inequality}, for any $t \in [0, T]$, the condition number of the smoothed density $p_t$ can be bounded pointwise.
Using $m(\gamma^d) = M(\gamma^d) = 1$ and $m_t, M_t$ defined the lower and upper bounds of \eqref{eq:pl_inequality_quantitative}, the condition number $\kappa_t = M_t/m_t$ evaluates to
\begin{align*}
    \kappa_t \coloneqq \frac{M_t}{m_t}
    &= \left(\frac{M_\star}{m_\star} \right) \frac{1 + m_\star \sigma_t^2 / \alpha_t^2}{1 + M_\star \sigma_t^2 / \alpha_t^2} .
\end{align*}
The condition number $\kappa_t$ is bounded uniformly over time (over all values of SNR $\alpha_t^2/\sigma_t^2$), and consequently there exists a global constant $\kappa = M_\star / m_\star$ such that:
\begin{equation*}
        \lVert \nabla s_t \rVert^2_{L_2} \leq \kappa^2 \frac{\fisher_t^2}{d}  , \quad \forall t \in [0, T] .
\end{equation*}

\textbf{Gaussian Mixture Models.}
Recall that we assume that the data distribution is a mixture of $L$ Gaussian components.
Specifically, $p_\star(x)=\sum_{i=1}^L \pi_i \varphi(x;\mu_i, \nu^2 I)$,
where $\varphi(x;\mu_i, \nu^2 I)$ is the density of $\mathcal{N}(\mu_i, \nu^2 I)$.

From the solution of the forward process $X_t \overset{d}{=} \alpha_t X_\star + \sigma_t Z$, it holds that the law of $X_t$ is a GMM with density
\begin{equation}\label{eq:gmm_t_density}
    p_t (x) = \sum_{i=1}^L \pi_i \varphi(x; \alpha_t \mu_i, v_t^2 I),
    \quad
    v_t^2 = \alpha_t^2 \nu^2 + \sigma_t^2 .
\end{equation}
Define the random variable $M$ which takes the value $\mu_i$ with probability $\pi_i$.
Let $\tilde{Z}$ be a standard Gaussian vector independent of $Z$ and $M$, so that
\begin{equation}
    X_t \overset{d}{=} \alpha_t M + v_t \tilde{Z} .
\end{equation}
Observe that $(\alpha_t M \mid X_t = x)$ takes values in $\{\alpha_t \mu_1, \cdots, \alpha_t \mu_L\}$, where $(\alpha_t M \mid X_t = x)$ is $\alpha_t \mu_i$ with probability
\begin{equation}\label{eq:gmm_posterior_prob}
    \frac{\pi_i \varphi(x; \alpha_t \mu_i, v_t^2 I)}{\sum_{j=1}^L \pi_j \varphi(x; \alpha_t \mu_j, v_t^2 I)}
    .
\end{equation}
The score function for GMMs can be evaluated using Tweedie's formula as in \citep[Section 3.2]{gatmiry2024learning}, where
\begin{align*}
    \nabla \log p_{X_t} (x) = \frac{1}{v_t^2} \left(\mathbb{E}[\alpha_t M \mid X_t = x] - x\right) .
\end{align*}
The Jacobian of the conditional expectation with respect to $x$ is then
\begin{align*}
    \nabla \mathbb{E}[\alpha_t M \mid X_t = x]
    =
    \frac{1}{v_t^2} \mathrm{Cov}(\alpha_t M \mid X_t = x)
    =
    \frac{\alpha_t^2}{v_t^2} \mathrm{Cov}(M \mid X_t = x)
    .
\end{align*}
Therefore, the second order Tweedie's formula evaluates to
\begin{align*}
    - \nabla^2 \log p_t (x) = \frac{1}{v_t^2} I - \frac{\alpha_t^2}{v_t^4} \mathrm{Cov}(M \mid X_t = x)
\end{align*}
and we use integration by parts as in the proof of \Cref{lem:st_Fisher_lb} to obtain
\begin{align*}
    \fisher_t = - \mathbb{E} \tr  \nabla^2 \log p_t (X_t)
    = \frac{d}{v_t^2} - \frac{\alpha_t^2}{v_t^4} \mathbb{E} \tr \mathrm{Cov}(M \mid X_t) .
\end{align*}
From $\fisher_t \geq 0$, we have
\begin{equation}\label{eq:gmm_fisher_1}
    \frac{\alpha_t^2}{v_t^2} \mathbb{E} \tr \mathrm{Cov}(M \mid X_t) \leq d .
\end{equation}
Using the integration by parts $\mathbb{E} \lVert s_t \rVert^2 = - \mathbb{E} \tr \nabla s_t$, we obtain
\begin{align*}
    \lVert \nabla s_t \rVert^2_{L_2(p_t)}
    &=
    \mathbb{E} \tr \left( - \nabla^2 \log p_t \right)^2
    \\ &
    = \frac{d}{v_t^4} - \frac{2 \alpha_t^2}{v_t^6} \mathbb{E} \tr \mathrm{Cov}(M \mid X_t)
    + \frac{\alpha_t^4}{v_t^8} \mathbb{E} \tr \left(\mathrm{Cov}(M \mid X_t)\right)^2
    \\ &
    \leq
    \frac{d}{v_t^4} + \frac{\alpha_t^4}{v_t^8} \mathbb{E} \tr \left(\mathrm{Cov}(M \mid X_t)\right)^2
    .
    \numberthis \label{eq:gmm_score_ub_1}
\end{align*}

By the semi-positive definiteness of $\mathrm{Cov}(M \mid X_t = x)$, there exists some $\lambda_{\max}$ such that $\mathrm{Cov}(M \mid X_t = x) \preceq \lambda_{\max} I$.
Multiplying both sides by $\mathrm{Cov}(M \mid X_t = x)$ and taking the trace, we obtain
\begin{align*}
    \tr \left(\mathrm{Cov}(M \mid X_t = x)\right)^2 \leq \lambda_{\max} \tr \mathrm{Cov}(M \mid X_t = x) .
\end{align*}
which implies
\begin{equation}\label{eq:gmm_expected_trace_bound_1}
    \mathbb{E}\tr \left(\mathrm{Cov}(M \mid X_t = x)\right)^2
    \leq
    \lambda_{\max} d \frac{v_t^2}{\alpha_t^2}
    .
\end{equation}

We next obtain an upper bound on $\lambda_{\max}$.
Define $\bar{\mu}_x = \mathbb{E}[M \mid X_t = x]$.
For any unit vector $u \in \mathbb{R}^d$,
\begin{align*}
    u^\dagger \mathrm{Cov}(M \mid X_t = x) u
    &=
    \sum_{i=1}^L \mathbb{P}\left(M = \mu_i \mid X_t = x\right) u^\dagger (\mu_i - \bar{\mu}_x) (\mu_i  - \bar{\mu}_x)^\dagger u
    \\ &
    = \mathrm{Var}(u^\dagger M \mid X_t = x) .
\end{align*}
The random variable $u^\dagger M$ takes values in the set $\{u^\dagger \mu_1, u^\dagger \mu_2, \cdots, u^\dagger \mu_L\}$, and
\begin{align*}
    \lvert u^\dagger (\mu_i - \mu_\ell) \rvert \leq \lVert u \rVert \lVert \mu_i - \mu_\ell \rVert \leq \Delta_\mu .
\end{align*}
Using the definition $\lambda_{\max} = \sup_{\lVert u \rVert_2 = 1} u^\dagger \mathrm{Cov}(M \mid X_t = x) u$, we have
\begin{align*}
    \lambda_{\max} = \max_{u: \lVert u \rVert_2 = 1} u^\dagger \mathrm{Cov}(M \mid X_t = x) u
    \leq \frac{1}{4} \Delta_\mu^2 .
\end{align*}
Substituting for \eqref{eq:gmm_expected_trace_bound_1}, we obtain
\begin{align*}
    \mathbb{E}\tr \left(\mathrm{Cov}(M \mid X_t = x)\right)^2
    \leq
    \frac{d}{4} \frac{v_t^2}{\alpha_t^2} \Delta_\mu^2
    .
\end{align*}
Substituting for \eqref{eq:gmm_score_ub_1}, we obtain
\begin{equation}\label{eq:gmm_score_ub_2}
    \lVert \nabla s_t \rVert_{L_2}^2
    \leq
    \frac{d}{v_t^4}
    +
    \frac{\alpha_t^4}{v_t^8} \frac{d}{4} \frac{v_t^2}{\alpha_t^2} \Delta_\mu^2
    =
    \frac{d}{v_t^4} \left(1 + \frac{1}{4} \frac{\alpha_t^2}{v_t^2} \Delta_\mu^2 \right) .
\end{equation}
A lower bound on $\fisher_t$ is obtained using the CRLB \eqref{eq:crlb}:
\begin{align*}
    \fisher_t \geq \frac{d^2}{\tr \mathrm{Cov}(X_t)}
    &= \frac{d^2}{\tr \left\{ \alpha_t^2 \mathrm{Cov}(M) + v_t^2 I \right\}}
    .
\end{align*}
Substituting
\begin{align*}
    \mathrm{Cov}(M) = \sum_{i=1}^L \pi_i (\mu_i - \bar{\mu}) (\mu_i - \bar{\mu})^\dagger ,
\end{align*}
we obtain
\begin{equation}\label{eq:gmm_fisher_lower_bound}
    \fisher_t \geq
    \frac{d^2}{\alpha_t^2 \sum_{i=1}^L \pi_i \lVert \mu_i - \bar{\mu} \rVert^2 + v_t^2 d}
    .
\end{equation}
Combining \eqref{eq:gmm_score_ub_2} and \eqref{eq:gmm_fisher_lower_bound}, we obtain
\begin{align*}
    \frac{d \lVert \nabla s_t \rVert^2_{L_2}}{\fisher_t^2}
    &\leq
    \frac{d^2}{v_t^4} \left(1 + \frac{1}{4} \frac{\alpha_t^2}{v_t^2} \Delta_\mu^2 \right)
    \cdot
    \left(\frac{\alpha_t^2 \sum_{i=1}^L \pi_i \lVert \mu_i - \bar{\mu} \rVert^2 + v_t^2 d}{d^2}\right)^2
    \\ &
    =
    \left(1 + \frac{1}{4} \frac{\alpha_t^2}{v_t^2} \Delta_\mu^2 \right)
    \left(1 + \frac{\alpha_t^2}{v_t^2} \frac{1}{d} \sum_{i=1}^L \pi_i \lVert \mu_i - \bar{\mu} \rVert^2\right)^2
    .
\end{align*}
The bound is further simplified by using $v_t^2 \geq \alpha_t^2 \nu^2$, which yields
\begin{align*}
    \frac{d \lVert \nabla s_t \rVert^2_{L_2}}{\fisher_t^2}
    \leq
    \left(1 + \frac{1}{4 \nu^2} \Delta_\mu^2 \right)
    \left(1 + \frac{1}{\nu^2} \frac{1}{d} \sum_{i=1}^L \pi_i \lVert \mu_i - \bar{\mu} \rVert^2\right)^2
    .
\end{align*}
The upper bound in \eqref{eq:kappa-upper-gmm} is obtained by observing that $\max_{1 \leq i, \ell \leq L} \lVert \mu_i - \bar{\mu} \rVert \leq R_\mu$ implies $\Delta_\mu^2 \leq 4 R_\mu^2$.

\end{proof}

\subsection{Proof of \texorpdfstring{\Cref{thm:schedules}}{Theorem 3}}\label{proof:admissable_g_robustness}
Let $g=\appx_q$.
By \Cref{lem:fisher_ode,lem:score_fisher_equivalence},
\begin{equation}\label{eq:disc_error_slc}
    h\int_0^T g(t)^2\lVert\nabla s_t\rVert_{L_2}^2\,dt
    \leq
    \frac{h\kappa^2}{d}\int_0^T[g(t)\fisher_t]^2dt.
\end{equation}
The upper bound on $\fisher_t^2/\lVert\nabla s_t\rVert_{L_2}^2$ in \Cref{lem:score_fisher_equivalence} also gives
\begin{equation}\label{eq:fisher_ode_ub_slc}
    \dot\fisher_t\leq 2f(t)\fisher_t-\frac{g(t)}{d}\fisher_t^2.
\end{equation}
Since $g'(t)/g(t)=q(t)-2f(t)$, \eqref{eq:fisher_ode_ub_slc} implies
\[
    \frac{d}{dt}\bigl(g(t)\fisher_t\bigr)
    \leq q(t)g(t)\fisher_t-\frac{\bigl(g(t)\fisher_t\bigr)^2}{d}.
\]
Using $\sup_{t \in [0, T] } q(t) \leq C \log (1 + n)$, we obtain
\[
    \frac{d}{dt}\bigl(g(t)\fisher_t\bigr)
    \leq
    C \log (1 + n)g(t)\fisher_t-\frac{\bigl(g(t)\fisher_t\bigr)^2}{d}.
\]
Whenever $g(t) \fisher_t \geq C d \log (1 + n)$, the upper bound is negative.
Since $g_0\fisher_\star\leq Cd\log(1+n)$, we have
\[
    g(t)\fisher_t\leq Cd\log(1+n),\qquad t\in[0,T].
\]
Since $q(t)\geq 2f(t)$, \eqref{eq:admissable_g} implies that $g$ is nondecreasing.
For each reverse-time step, \eqref{eq:score_increment_before_absorption} and $g(T-\tau)\leq g(T-\tilde\tau)$ for $\tilde\tau\leq\tau$ give
\begin{align*}
&\int_{\tau_k}^{\tau_{k+1}}g(T-\tau)
\left\lVert s_{T-\tau}(X_\tau^\leftarrow)-s_{T-\tau_k}(X_{\tau_k}^\leftarrow)\right\rVert_{L_2}^2d\tau
\\ &\lesssim
h\int_{\tau_k}^{\tau_{k+1}}\int_{\tau_k}^{\tau}
g(T-\tilde\tau)f(T-\tilde\tau)^2\fisher_{T-\tilde\tau}\,d\tilde\tau\,d\tau
+
\int_{\tau_k}^{\tau_{k+1}}\int_{\tau_k}^{\tau}
g(T-\tilde\tau)^2\left\lVert\nabla s_{T-\tilde\tau}\right\rVert_{L_2}^2\,d\tilde\tau\,d\tau
\\ &\leq
h^2\int_{\tau_k}^{\tau_{k+1}}g(T-\tilde\tau)f(T-\tilde\tau)^2\fisher_{T-\tilde\tau}\,d\tilde\tau
+
h\int_{\tau_k}^{\tau_{k+1}}g(T-\tilde\tau)^2
\left\lVert\nabla s_{T-\tilde\tau}\right\rVert_{L_2}^2\,d\tilde\tau.
\end{align*}
Summing over $k$ and changing variables to forward time bounds the discretization error in \eqref{eq:girsanov} by
\[
    h^2\int_0^T g(t)f(t)^2\fisher_t\,dt
    +h\int_0^T g(t)^2\lVert\nabla s_t\rVert_{L_2}^2\,dt.
\]
The first term is bounded using $f(t)\leq q(t)/2\leq C\log(1+n)/2$ and $g(t)\fisher_t\leq Cd\log(1+n)$.
For the second term, \Cref{lem:score_fisher_equivalence} gives
\[
    h\int_0^T g(t)^2\lVert\nabla s_t\rVert_{L_2}^2\,dt
    \leq \frac{h\kappa^2}{d}\int_0^T[g(t)\fisher_t]^2dt.
\]
Thus, with $h=T/n$, the discretization error is bounded by a constant multiple of
\[
    \frac{\kappa^2 C^2 dT^2\log^2(1+n)}{n}
    +\frac{C^3dT^3\log^3(1+n)}{n^2}.
\]
For the initialization error, \eqref{eq:admissable_g} and the second condition $
    \int_0^T\exp\left(\int_0^t q(\tilde t)\,d\tilde t\right)dt
    \geq\frac{c}{g_0}\frac{\lVert X_\star\rVert_{L_2}^2}{d}\frac{n}{\log(1+n)}
$ assumed in \Cref{thm:schedules} give
\[
    \frac{\sigma_T^2}{\alpha_T^2}
    =g_0\int_0^T\exp\left(\int_0^t q(\tilde t)\,d\tilde t\right)dt
    \geq c\frac{\lVert X_\star\rVert_{L_2}^2 n}{d\log(1+n)}.
\]
Consequently, $\alpha_T^2\lVert X_\star\rVert_{L_2}^2/\sigma_T^2
\leq d\log(1+n)/(cn)$.
Combining these bounds with \eqref{eq:girsanov}, we obtain
\[
  \mathrm{KL}(p_\star\|\hat p_\star)
  \lesssim \frac{d}{n}
  \left[
  \frac{\log(1+n)}{c}
  +\kappa^2C^2T^2\log^2(1+n)
  +\frac{C^3T^3\log^3(1+n)}{n}
  \right].
\]

\subsection{Proof of \texorpdfstring{\Cref{cor:optimal_schedules_practice}}{the ACS Corollary}}\label{proof:optimal_schedules_practice}
Differentiating \eqref{eq:admissable_g} and substituting $f(t)=\theta\appx_q(t)+\omega$ gives
\[
    \frac{d}{dt}\appx_q(t)=(q(t)-2\omega)\appx_q(t)-2\theta\appx_q(t)^2.
\]
Consequently,
\[
    \frac{d}{dt}\left[
    \frac{\exp\left(\int_0^t(q(\tilde t)-2\omega)\,d\tilde t\right)}
    {\appx_q(t)}
    \right]
    =2\theta\exp\left(\int_0^t(q(\tilde t)-2\omega)\,d\tilde t\right).
\]
Integrating from $0$ to $t$ and using $\appx_q(0)=g_0$ yields \eqref{eq:optimal_schedules_practice}.

\subsection{Verifying \texorpdfstring{\Cref{thm:schedules}}{Theorem 3} for VP Schedules}\label{sec:connections}
\paragraph{VP-Sigmoid.}
For $q(t)=\bar q=2\log(1+n)/T$ and $g_0>0$ independent of $n$, we have $g_0<\bar q$ for sufficiently large $n$.
The schedule in \eqref{eq:optimal_schedules_practice} then satisfies $2f(t)=\appx_q(t)<\bar q=q(t)$ for all $t\in[0,T]$, and
\[
    \sup_{t\in[0,T]}q(t)=\frac{2}{T}\log(1+n).
\]
Moreover,
\[
    \int_0^T\exp\left(\int_0^t q(\tilde t)\,d\tilde t\right)dt
    =\frac{e^{\bar qT}-1}{\bar q}
    =\frac{Tn(n+2)}{2\log(1+n)}.
\]
For sufficiently large $n$, this is larger than the lower bound required in \Cref{thm:schedules}, so the conditions on $q$ in that theorem hold.

\paragraph{VP-linear.}
For $g_0>0$ independent of $n$, let
\[
    g_{\min}=g_0,
    \qquad
    g_{\max}=g_{\min}+\frac{2}{T}\log(1+n)
    ,
\]
and define
\begin{equation}
    g(t)=g_{\min}+\frac{g_{\max}-g_{\min}}{T}t,
    \qquad f(t)=\frac{g(t)}{2},
    \qquad q(t)=g(t)+\frac{g_{\max}-g_{\min}}{Tg(t)}.
\end{equation}
Then $g'(t)/g(t)=q(t)-2f(t)$, and \eqref{eq:admissable_g} recovers the VP-linear schedule.
Since $g'(t)>0$, we have $q(t)\geq g(t)=2f(t)$, and
\[
    \sup_{t\in[0,T]}q(t)
    \leq g_{\max}+\frac{g_{\max}-g_{\min}}{Tg_{\min}}
    =\mathcal O(\log(1+n)).
\]
Moreover, $q(t)=g(t)+g'(t)/g(t)$ implies
\[
    \int_0^T\exp\left(\int_0^t q(\tilde t)\,d\tilde t\right)dt
    =\frac{e^{\int_0^T g(t)\,dt}-1}{g_{\min}}
    =\frac{e^{Tg_{\min}}(1+n)-1}{g_{\min}}.
\]
For sufficiently large $n$, this is larger than the lower bound required in \Cref{thm:schedules}, so the conditions on $q$ required to achieve the $\tilde{\mathcal O}(d/n)$ rate hold.

\section{Error Bound for the Constant Schedule}\label{app:constant_error_bound}

\begin{proposition}\label{prop:constant_error_bound}
    Suppose \Cref{ass:fisher_information,ass:surrogate} hold.
    For sufficiently small $h$, the VP-constant schedule $(f_{const},g_{const})$ with $f_{const}=g_{const}/2$ achieves the error bound
    \begin{equation}
        \mathrm{KL}(p_\star\|\hat p_\star)
        \lesssim
        \frac{\alpha_T^2}{2\sigma_T^2}\lVert X_\star\rVert^2_{L_2}
        +hd\kappa^2g_{const}\left[
            \frac{\fisher_\star}{d}
            +\log\left(1+\frac{\fisher_\star}{d}(e^{g_{const}T}-1)\right)
        \right].
    \end{equation}
\end{proposition}
\begin{proof}
    See \Cref{proof:constant_error_bound}.
\end{proof}

As discussed earlier, the best bounds \citep{chen2022sampling,conforti2024klconvergenceguaranteesscore} available for VP-constant schedules grow linearly with $\fisher_\star/d$.
We show that, under our analysis, a constant schedule suffers from the larger rate even when the constants are chosen to scale with problem constants as in \Cref{thm:schedules}.
Let $T=1$ and define $\mathcal E=\int_0^1 f(t)dt$.
For the constant schedule, $g_{const}=2\mathcal E$, and \Cref{prop:constant_error_bound} gives
\begin{equation}\label{eq:constant_bound_as_energy}
    \mathrm{KL}(p_\star\|\hat p_\star)
    \lesssim
        \frac{\lVert X_\star\rVert_{L_2}^2}{2(e^{g_{const}}-1)}
    +
        \frac{d\kappa^2}{n}g_{const}\left[
            \frac{\fisher_\star}{d}
            +\log\left(1+\frac{\fisher_\star}{d}(e^{g_{const}}-1)\right)
        \right]
\end{equation}
The choice of $g_{const}^*$ that minimizes the bound yields
\begin{equation}
    \mathrm{KL}(p_\star\|\hat p_\star)
    \lesssim
    \frac{\kappa^2}{n} \left(d \log^2 n + \fisher_\star \log n \right)
    .
\end{equation}
Therefore, under our framework, the constant schedule cannot eliminate the linear dependence on $\fisher_\star / d$ even with a problem-dependent choice of $g_{const}$.

\subsection{Proof of \texorpdfstring{\Cref{prop:constant_error_bound}}{the Constant-Schedule Error Bound}}\label{proof:constant_error_bound}
Under the VP constraint, \eqref{eq:fisher_ode_ub_slc} gives
\begin{equation*}
    \dot\fisher_t
    \leq
    g_{const}\fisher_t\left(1-\frac{\fisher_t}{d}\right).
\end{equation*}
Solving this differential inequality yields
\begin{equation*}
    \fisher_t
    \leq
    \frac{d e^{g_{const}t}}{e^{g_{const}t}-(1-d/\fisher_\star)}.
\end{equation*}
Therefore,
\begin{equation}\label{eq:fisher_integral_ub}
    \int_0^T\fisher_tdt
    \leq
    \frac{d}{g_{const}}
    \log\left(1+\frac{\fisher_\star}{d}(e^{g_{const}T}-1)\right).
\end{equation}
Applying the integration-by-parts argument for the constant schedule gives
\begin{equation}
    \int_0^T(g_{const}\fisher_t)^2dt
    \leq
    dg_{const}\fisher_\star
    +dg_{const}^2\int_0^T\fisher_tdt.
\end{equation}
Substituting \eqref{eq:fisher_integral_ub}, we obtain \Cref{prop:constant_error_bound}.

\section{Supplement to Experiments (\texorpdfstring{\Cref{sec:experiments}}{Experiments})}\label{app:experiments}
The time-varying function $q$ is parametrized by a rational cubic Bernstein basis.
We set $T=1$ and use the schedule in \Cref{cor:optimal_schedules_practice}.
Let
\begin{equation}\label{eq:experimental_bernstein}
    r_{\ell,\epsilon}(t)
    =
    \frac{
        \sum_{j=0}^3 e^{\ell_j}\binom{3}{j}t^j(1-t)^{3-j}
    }{
        (t+\epsilon)(1-t+\epsilon_{\mathrm{term}})
    },
    \qquad
    \psi(t)
    =
    \gamma\frac{r_{\ell,\epsilon}(t)}{\int_0^1 r_{\ell,\epsilon}(s)ds},
\end{equation}
where $\ell_0=0$ and $\epsilon_{\mathrm{term}}=10^{-6}$.
Thus, $\psi$ is positive and satisfies $\int_0^1\psi(t)dt=\gamma$.
For an NFE budget $n$, define
\begin{equation}\label{eq:experimental_acs_parameters}
    \lambda_n=\mathcal W(K n),
    \qquad
    \mathcal E_n=\frac{\lambda_n}{2},
    \qquad
    q_n(t)=\lambda_n\psi(t),
    \qquad
    \omega_n=\rho\mathcal E_n.
\end{equation}
In the implementation, we set $K = \tfrac{2\mathcal E_{10}\exp(2\mathcal E_{10})}{10}$ which ensures that $\lambda_{10}=2\mathcal E_{10}$.
The initial diffusion coefficient is determined by the energy constraint $\int_0^1 f_n(t)dt=\mathcal E_n$:
\begin{equation}\label{eq:experimental_g0}
    g_{0,n}
    =
    \frac{\exp\bigl(2(\mathcal E_n-\omega_n)\bigr)-1}
    {2\theta\int_0^1
        \exp\left(\int_0^s(q_n(u)-2\omega_n)du\right)ds}.
\end{equation}
Substituting $q_n$, $\omega_n$, and $g_{0,n}$ into \eqref{eq:optimal_schedules_practice} determines $(f_n,\appx_{q_n})$ for every $n$.

\paragraph{Hyperparameter selection.}
We selected a single parameter vector for each schedule by minimizing the mean log-FID over CIFAR-10 and ImageNet $64\times64$ at NFE 10 and 50.
Each objective evaluation used 10,000 generated samples.
The universal selection study used 32 trials per schedule in Optuna \citep{akiba2019optuna}, with scrambled Sobol points for initialization and CMA-ES for multidimensional search spaces.
After selection, we fixed every schedule's parameters across all NFE budgets and evaluation datasets.

\begin{figure}[htbp]
    \centering
    \includegraphics[width=0.85\linewidth]{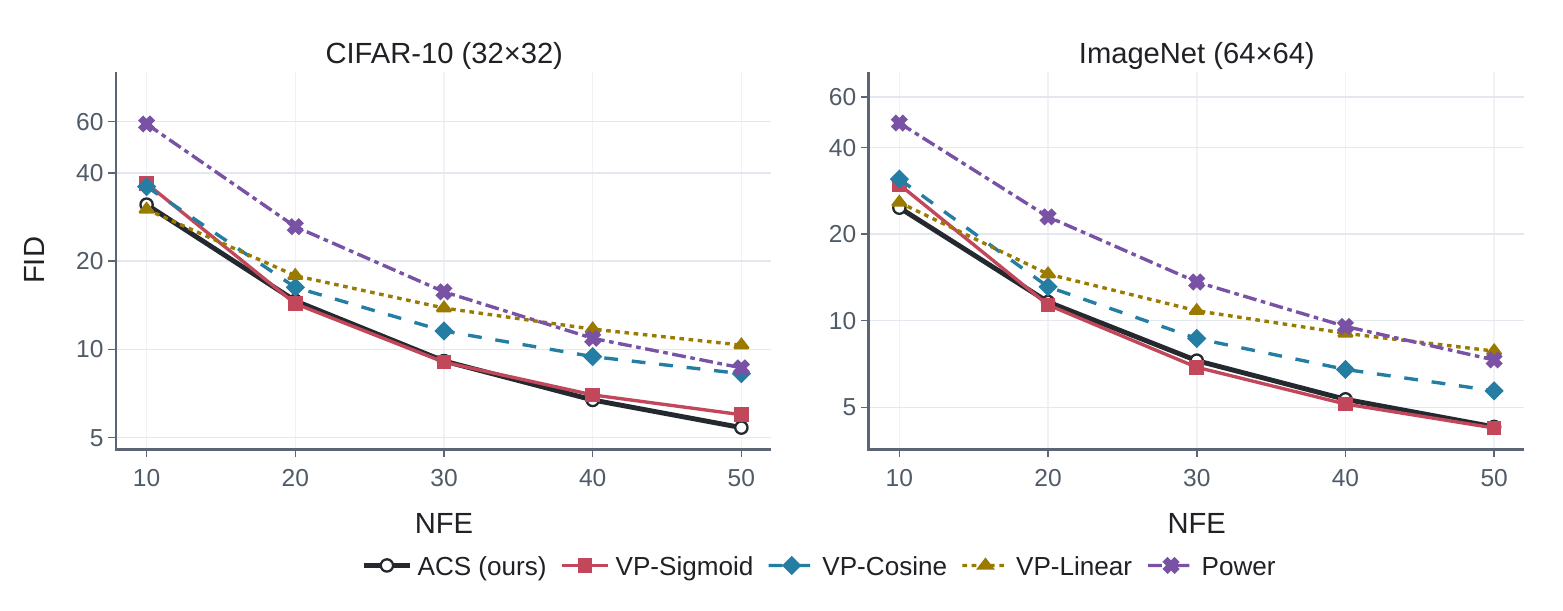}
    \caption{
        FID on 50,000 generated samples across inference budgets on CIFAR-10 and ImageNet $64\times64$, the datasets used for hyperparameter selection.
        Hyperparameters were selected using 10,000-sample FID at NFE 10 and 50 and then fixed across all inference budgets.
    }
    \label{fig:fid50k_selection}
\end{figure}

\begin{table}[ht]
\centering
\caption{
    Search spaces and selected parameters used for the generation experiments.
    Logarithmic search spaces are indicated by ``log.''
}\label{tab:hparam_decision}
\begin{tabular}{llll}
\toprule
    Schedule & Parameter & Search space & Selected value \\
    \midrule
    VP-Linear & $\beta_{\min}$ & $[10^{-4},0.5]$ (log) & $0.003318$ \\
    & $\beta_{\max}$ & $[1,50]$ (log) & $9.400954$ \\
    \midrule
    VP-Cosine & offset & $[10^{-5},0.1]$ (log) & $1.778279\times10^{-5}$ \\
    \midrule
    VP-Sigmoid & start & $[-6,-1]$ & $-4.576358$ \\
    & end & $[1,6]$ & $3.408840$ \\
    & temperature & $[0.25,2.5]$ (log) & $0.765914$ \\
    \midrule
    ACS & $\mathcal E_{10}$ & $[0.0456,2.625]$ & $1.634306$ \\
    & $\gamma$ & $[1,8]$ (log) & $4.443826$ \\
    & $\theta$ & $[10^{-3},1]$ (log) & $0.012058$ \\
    & $\rho$ & $[0,0.99]$ & $0.515484$ \\
    & $\epsilon$ & $[10^{-5},0.2]$ (log) & $0.016110$ \\
    & $\ell_1$ & $[-10,5]$ & $-4.330955$ \\
    & $\ell_2$ & $[-10,5]$ & $1.366422$ \\
    & $\ell_3$ & $[-10,5]$ & $-0.586175$ \\
\bottomrule
\end{tabular}
\end{table}

\paragraph{Generation and evaluation.}
All reverse processes were initialized with $X_0^\leftarrow\sim\mathcal N(0,\sigma_T^2I)$.
We used ancestral sampling \citep{ho2020denoising} with pretrained EDM models for CIFAR-10, ImageNet $64\times64$, FFHQ $64\times64$, and AFHQv2 $64\times64$ \citep{karrasElucidatingDesignSpace2022}, and a pretrained EDM2-XS model \citep{karras2024analyzing} for ImageNet $512\times512$.
We applied each schedule to the corresponding pretrained model through the statistical equivalence in \Cref{app:score_function_statistical_equivalence}, without fine-tuning the model.

\Cref{fig:grid_comparison} provides a qualitative comparison of samples generated for ImageNet $512 \times 512$.
The visual quality of the baseline schedules varies across examples, whereas ACS consistently produces samples comparable in quality to the strongest baseline in each example.

\begin{figure}[p]
    \centering
    \includegraphics[width=\linewidth,height=0.82\textheight,keepaspectratio]{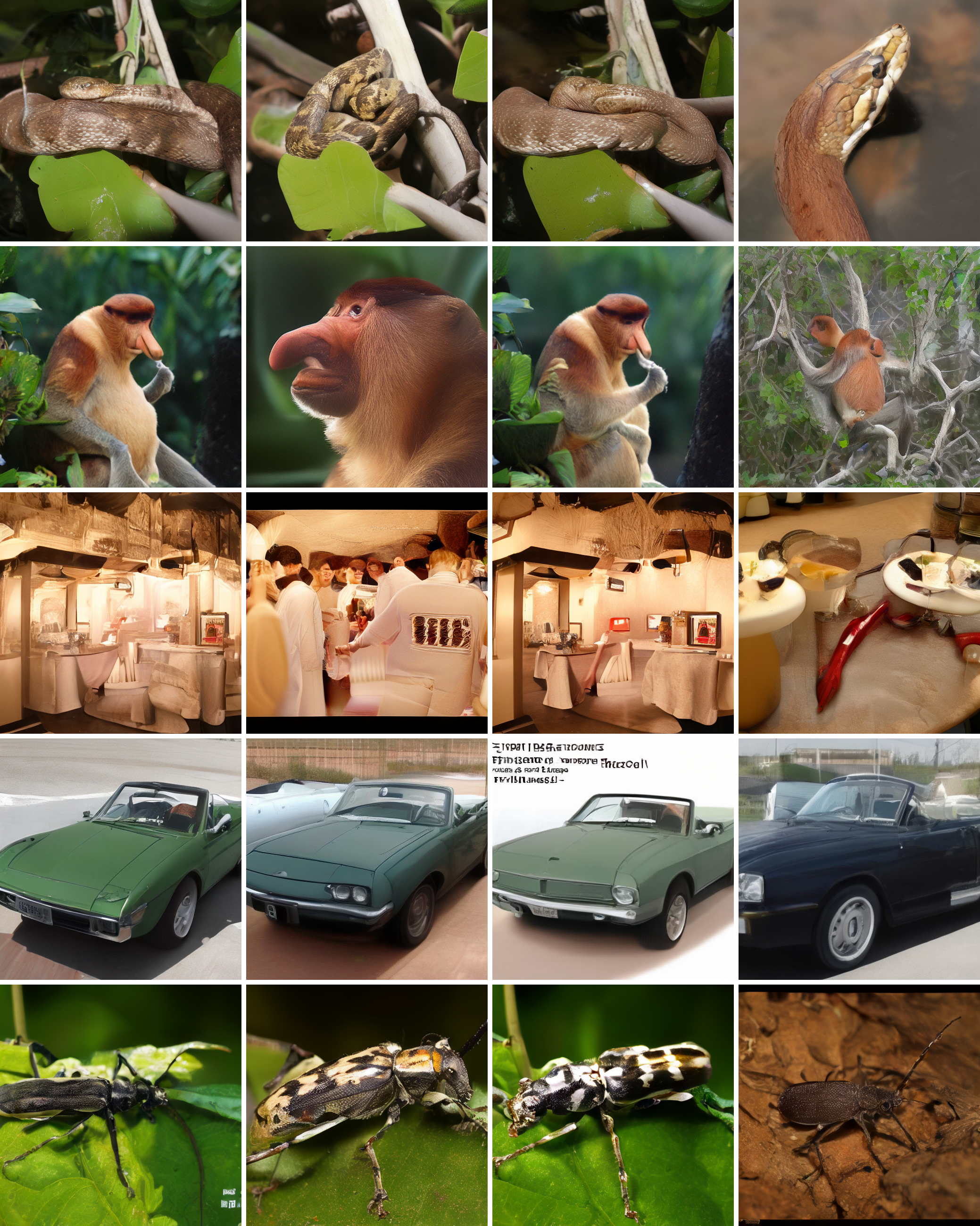}
    \caption{ImageNet $512 \times 512$ samples generated at NFE 50.
    From left to right, the schedules are ACS, VP-Linear, VP-Sigmoid, and Power.
    The rows are in order of: ``Water snake'', ``Proboscis monkey'', ``Restaurant'', ``Convertible'', and ``Long-horned beetle''.
    The samples generated using VP-Cosine were excluded because they were illegible for the ImageNet $512 \times 512$ dataset.
    }
    \label{fig:grid_comparison}
\end{figure}

\Cref{fig:all_noise_schedules} plot $(f(t),g(t))$ and $(\alpha_t,\sigma_t)$ for ACS, VP-Linear, VP-Cosine, and VP-Sigmoid using the parameters selected for the generation experiments.
We show the schedules at NFEs 10 and 50 because the ACS parameters depend on the NFE budget.
\begin{figure}[htbp]
    \centering
    \begin{subfigure}{\textwidth}
        \centering
        \includegraphics[width=0.9\linewidth]{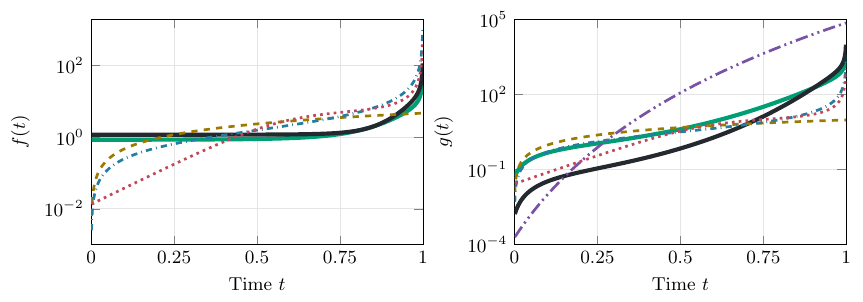}
        \label{fig:schedules_fg}
    \end{subfigure}
    \vspace{1em}
    \begin{subfigure}{\textwidth}
        \centering
        \includegraphics[width=0.9\linewidth]{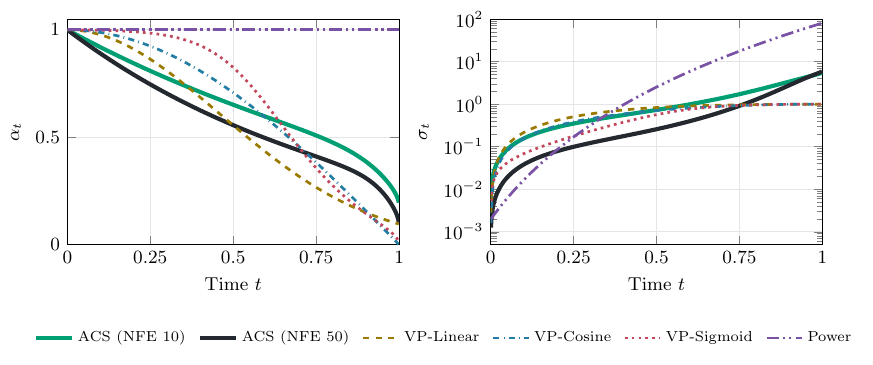}

        \label{fig:schedules_alphasigma}
    \end{subfigure}
    \caption{
        Noise schedules used in the generation experiments.
        The upper plots show logarithmic scales of $f(t)$ and $g(t)$, with Power's $f(t) = 0$ excluded.
        The lower plots show $\alpha_t$ and $\sigma_t$.
        All curves are evaluated on $t \in [10^{-3},1-10^{-3}]$.
    }
    \label{fig:all_noise_schedules}
\end{figure}

\end{document}